\documentclass{article}

\usepackage[utf8]{inputenc} %
\usepackage[T1]{fontenc}    %
\usepackage[colorlinks=true, allcolors=blue]{hyperref}       %
\usepackage{url}            %
\usepackage{booktabs}       %
\usepackage{amsfonts}       %
\usepackage{nicefrac}       %
\usepackage{microtype}      %
\usepackage{xcolor}         %
\usepackage{geometry}
\usepackage{amsmath, amsthm, amsfonts, amssymb}
\usepackage{appendix}
\usepackage{parskip}

\PassOptionsToPackage{hypertexnames=false}{hyperref}  %
\usepackage{parskip}

\usepackage{microtype}
\usepackage{hhline}

\makeatletter
\newcommand{\neutralize}[1]{\expandafter\let\csname c@#1\endcsname\count@}
\makeatother

\usepackage{algorithm}

\usepackage{natbib}
\bibpunct{(}{)}{;}{a}{,}{,}

\usepackage{amsthm}
\usepackage{mathtools}
\usepackage{amsmath}
\usepackage{bbm}
\usepackage{amsfonts}
\usepackage{amssymb}

\usepackage{xpatch}

\usepackage{thmtools}
\usepackage{thm-restate}
\declaretheorem[name=Theorem,parent=section]{theorem}
\declaretheorem[name=Lemma,parent=section]{lemma}
\declaretheorem[name=Assumption, parent=section]{assumption}
\declaretheorem[name=Condition, parent=section]{condition}
\declaretheorem[qed=$\triangleleft$,name=Example,style=definition, parent=section]{example}

\declaretheorem[name=Proposition, parent=section]{proposition}

\makeatletter
  \renewenvironment{proof}[1][Proof]%
  {%
   \par\noindent{\bfseries\upshape {#1.}\ }%
  }%
  {\qed\newline}
  \makeatother

\theoremstyle{definition}  %

\newtheorem{corollary}{Corollary}[section]

\theoremstyle{plain}
\newtheorem{definition}{Definition}[section]

\xpatchcmd{\proof}{\itshape}{\normalfont\proofnameformat}{}{}
\newcommand{\proofnameformat}{\bfseries}

\usepackage[nameinlink,capitalize]{cleveref}

\newcommand{\pfref}[1]{Proof of \cref{#1}}
\newcommand{\savehyperref}[2]{\texorpdfstring{\hyperref[#1]{#2}}{#2}}

\Crefformat{figure}{#2Figure #1#3}
\Crefname{assumption}{Assumption}{Assumptions}
\Crefformat{assumption}{#2Assumption #1#3}

\usepackage{crossreftools}
\usepackage{prettyref}
\newcommand{\pref}[1]{\prettyref{#1}}
\newrefformat{eq}{\savehyperref{#1}{\textup{(\ref*{#1})}}}
\newrefformat{eqn}{\savehyperref{#1}{Equation~\ref*{#1}}}
\newrefformat{lem}{\savehyperref{#1}{Lemma~\ref*{#1}}}
\newrefformat{def}{\savehyperref{#1}{Definition~\ref*{#1}}}
\newrefformat{line}{\savehyperref{#1}{line~\ref*{#1}}}
\newrefformat{thm}{\savehyperref{#1}{Theorem~\ref*{#1}}}
\newrefformat{corr}{\savehyperref{#1}{Corollary~\ref*{#1}}}
\newrefformat{cor}{\savehyperref{#1}{Corollary~\ref*{#1}}}
\newrefformat{sec}{\savehyperref{#1}{Section~\ref*{#1}}}
\newrefformat{app}{\savehyperref{#1}{Appendix~\ref*{#1}}}
\newrefformat{ass}{\savehyperref{#1}{Assumption~\ref*{#1}}}
\newrefformat{ex}{\savehyperref{#1}{Example~\ref*{#1}}}
\newrefformat{fig}{\savehyperref{#1}{Figure~\ref*{#1}}}
\newrefformat{alg}{\savehyperref{#1}{Algorithm~\ref*{#1}}}
\newrefformat{rem}{\savehyperref{#1}{Remark~\ref*{#1}}}
\newrefformat{conj}{\savehyperref{#1}{Conjecture~\ref*{#1}}}
\newrefformat{prop}{\savehyperref{#1}{Proposition~\ref*{#1}}}
\newrefformat{proto}{\savehyperref{#1}{Protocol~\ref*{#1}}}
\newrefformat{prob}{\savehyperref{#1}{Problem~\ref*{#1}}}
\newrefformat{claim}{\savehyperref{#1}{Claim~\ref*{#1}}}

\usepackage{xparse}

\ExplSyntaxOn
\DeclareDocumentCommand{\XDeclarePairedDelimiter}{mm}
 {
  \__egreg_delimiter_clear_keys: %
  \keys_set:nn { egreg/delimiters } { #2 }
  \use:x %
   {
    \exp_not:n {\NewDocumentCommand{#1}{sO{}m} }
     {
      \exp_not:n { \IfBooleanTF{##1} }
       {
        \exp_not:N \egreg_paired_delimiter_expand:nnnn
         { \exp_not:V \l_egreg_delimiter_left_tl }
         { \exp_not:V \l_egreg_delimiter_right_tl }
         { \exp_not:n { ##3 } }
         { \exp_not:V \l_egreg_delimiter_subscript_tl }
       }
       {
        \exp_not:N \egreg_paired_delimiter_fixed:nnnnn 
         { \exp_not:n { ##2 } }
         { \exp_not:V \l_egreg_delimiter_left_tl }
         { \exp_not:V \l_egreg_delimiter_right_tl }
         { \exp_not:n { ##3 } }
         { \exp_not:V \l_egreg_delimiter_subscript_tl }
       }
     }
   }
 }

\keys_define:nn { egreg/delimiters }
 {
  left      .tl_set:N = \l_egreg_delimiter_left_tl,
  right     .tl_set:N = \l_egreg_delimiter_right_tl,
  subscript .tl_set:N = \l_egreg_delimiter_subscript_tl,
 }

\cs_new_protected:Npn \__egreg_delimiter_clear_keys:
 {
  \keys_set:nn { egreg/delimiters } { left=.,right=.,subscript={} }
 }

\cs_new_protected:Npn \egreg_paired_delimiter_expand:nnnn #1 #2 #3 #4
 {%
  \mathopen{}
  \mathclose\c_group_begin_token
   \left#1
   #3
   \group_insert_after:N \c_group_end_token
   \right#2
   \tl_if_empty:nF {#4} { \c_math_subscript_token {#4} }
 }
\cs_new_protected:Npn \egreg_paired_delimiter_fixed:nnnnn #1 #2 #3 #4 #5
 {
  \mathopen{#1#2}#4\mathclose{#1#3}
  \tl_if_empty:nF {#5} { \c_math_subscript_token {#5} }
 }
\ExplSyntaxOff

\XDeclarePairedDelimiter{\supnorm}{
  left=\lVert,
  right=\rVert,
  subscript=\infty
  }

\usepackage{framed}
\usepackage{comment}
\usepackage[noend]{algpseudocode}
\usepackage{verbatim}
\usepackage{transparent}
\usepackage{xargs}
\usepackage{enumitem}
\usepackage{xspace}
\usepackage{tikz-cd}
\usepackage{lipsum}
\usepackage{adjustbox}

\usepackage{prettyref}
\renewcommand{\pref}[1]{\prettyref{#1}}
\newrefformat{eq}{\savehyperref{#1}{Eq.~\textup{(\ref*{#1})}}}
\newrefformat{eqn}{\savehyperref{#1}{Equation~\ref*{#1}}}
\newrefformat{lem}{\savehyperref{#1}{Lemma~\ref*{#1}}}
\newrefformat{lemma}{\savehyperref{#1}{Lemma~\ref*{#1}}}
\newrefformat{def}{\savehyperref{#1}{Definition~\ref*{#1}}}
\newrefformat{line}{\savehyperref{#1}{Line~\ref*{#1}}}
\newrefformat{thm}{\savehyperref{#1}{Theorem~\ref*{#1}}}
\newrefformat{corr}{\savehyperref{#1}{Corollary~\ref*{#1}}}
\newrefformat{cor}{\savehyperref{#1}{Corollary~\ref*{#1}}}
\newrefformat{sec}{\savehyperref{#1}{Section~\ref*{#1}}}
\newrefformat{subsec}{\savehyperref{#1}{Section~\ref*{#1}}}
\newrefformat{app}{\savehyperref{#1}{Appendix~\ref*{#1}}}
\newrefformat{assum}{\savehyperref{#1}{Assumption~\ref*{#1}}}
\newrefformat{ex}{\savehyperref{#1}{Example~\ref*{#1}}}
\newrefformat{fig}{\savehyperref{#1}{Figure~\ref*{#1}}}
\newrefformat{alg}{\savehyperref{#1}{Algorithm~\ref*{#1}}}
\newrefformat{rem}{\savehyperref{#1}{Remark~\ref*{#1}}}
\newrefformat{conj}{\savehyperref{#1}{Conjecture~\ref*{#1}}}
\newrefformat{prop}{\savehyperref{#1}{Proposition~\ref*{#1}}}
\newrefformat{proto}{\savehyperref{#1}{Protocol~\ref*{#1}}}
\newrefformat{prob}{\savehyperref{#1}{Problem~\ref*{#1}}}
\newrefformat{claim}{\savehyperref{#1}{Claim~\ref*{#1}}}
\newrefformat{que}{\savehyperref{#1}{Question~\ref*{#1}}}
\newrefformat{op}{\savehyperref{#1}{Open Problem~\ref*{#1}}}
\newrefformat{fn}{\savehyperref{#1}{Footnote~\ref*{#1}}}
\newrefformat{tab}{\savehyperref{#1}{Table~\ref*{#1}}}
\newrefformat{fig}{\savehyperref{#1}{Figure~\ref*{#1}}}
\newrefformat{proc}{\savehyperref{#1}{Procedure~\ref*{#1}}}
\newrefformat{property}{\savehyperref{#1}{Property~\ref*{#1}}}

\newcommand{\cA}{\mathcal{A}}

\newcommand{\cC}{\mathcal{C}}
\newcommand{\cD}{\mathcal{D}}
\newcommand{\cE}{\mathcal{E}}

\newcommand{\cJ}{\mathcal{J}}

\newcommand{\cM}{\mathcal{M}}

\newcommand{\cS}{\mathcal{S}}
\newcommand{\cT}{\mathcal{T}}

\newcommand{\cX}{\mathcal{X}}

\newcommand{\bE}{\mathbb{E}}

\newcommand{\bP}{\mathbb{P}}
\newcommand{\bQ}{\mathbb{Q}}
\newcommand{\bR}{\mathbb{R}}

\newcommand{\En}{\mathbb{E}}

\newcommand{\vd}{\mathbf{d}}

\newcommand{\wt}[1]{\widetilde{#1}}
\newcommand{\wh}[1]{\widehat{#1}}
\newcommand{\wb}[1]{\widebar{#1}}

\newcommand{\veps}{\varepsilon}
\DeclareMathOperator*{\argmax}{arg\,max}

\newcommand{\ldef}{\vcentcolon=}

\newcommand{\Dhel}[2]{D_{\mathrm{H}}\prn*{#1,#2}}
\newcommand{\Dhels}[2]{D^{2}_{\mathrm{H}}\prn*{#1,#2}}

\newcommand{\poly}{\mathrm{poly}}
\newcommand{\Reg}{\mathrm{Reg}}

\DeclarePairedDelimiter{\abs}{\lvert}{\rvert} %
\DeclarePairedDelimiter{\brk}{[}{]}

\DeclarePairedDelimiter{\prn}{(}{)}

\DeclarePairedDelimiter{\inner}{\langle}{\rangle}
\DeclarePairedDelimiter{\set}{\{}{\}}

\def\medskip{\vskip 10 pt}
\def\bigskip{\vskip 15 pt}

\newcommand{\multiline}[1]{\parbox[t]{\dimexpr\linewidth-\algorithmicindent}{#1}}
\newcommand{\revindent}[1][1]{\hspace{#1in}&\hspace{-#1in}}

\def\texitem#1{\par\vspace{5pt}
\noindent\hangindent 20pt
\hbox to 20pt {\hss #1 ~}\ignorespaces}

\newcommand{\tto}{\Rightarrow}

\newcommand{\Mhat}{\wh{M}}
\newcommandx{\Mhatm}[1][1=m]{\wh{M}_{#1}}

\newcommand{\reg}{\mathrm{reg}}
\newcommand{\reghat}{\wh{\reg}}
\newcommand{\idx}[1]{_{{#1}}}
\newcommand{\ind}[1]{^{{#1}}}
\newcommand{\ocm}{d} %
\newcommand{\PiRNS}{\Pi_{\mathrm{RNS}}}

\newcommandx{\ocmm}[2][1=m,2=h]{\ocm\idx{#1}\ind{#2}}
\newcommandx{\ocmhat}[2][1=m,2=h]{\wh{\ocm}\idx{#1}\ind{#2}}
\newcommandx{\ocmtil}[2][1=m,2=h]{\wt{\ocm}\idx{#1}\ind{#2}}
\newcommandx{\gammam}[1][1=m]{\gamma\idx{#1}}
\newcommandx{\etam}[1][1=m]{\eta\idx{#1}}
\newcommandx{\lambdam}[3][1=m,2=h,3=c]{\lambda\idx{#1,#3}\ind{#2}}
\newcommandx{\reghatm}[1][1=m]{\reghat\idx{#1}}
\newcommandx{\metapol}[2][1=m,2=h]{p\idx{#1}\ind{#2}}
\newcommandx{\polCov}[2][1=m,2=h]{\Pi\idx{#1}\ind{#2}}
\newcommandx{\Phatm}[2][1=m,2=h]{\Phat\idx{#1}\ind{#2}}
\newcommandx{\zetam}[1][1=m]{\zeta\idx{#1}}

\newcommandx{\Ttil}[2][1=m,2=h]{\wt{\cT}\idx{#1}\ind{#2}}

\newcommandx{\pol}[5][1=m,2=c,3=h,4=s,5=a]{\pi\idx{#1,#2}\ind{#3,#4,#5}}
\newcommandx{\xim}[1][1=m]{\xi\idx{#1}}

\newcommand{\Vhat}{\wh{V}}
\newcommandx{\Vhatm}[2][1=m,2=1]{\Vhat\idx{#1}\ind{#2}}
\newcommandx{\Vstar}[1][1=1]{V_\star\ind{#1}}
\newcommandx{\Vm}[2][1=M,2=1]{V_{#1}\ind{#2}}
\newcommandx{\pistar}[1][1=c]{\pi_{\star,#1}}
\newcommand{\Phat}{\wh{P}}
\newcommand{\Rhat}{\wh{R}}

\tikzcdset{multi/.style 2 args={phantom, "#1"{name=#2, inner sep=1ex}}}

\newcommandx{\Rhatm}[2][1=m,2=h]{\Rhat_{#1}^{#2}}

\newcommand{\Mbar}{\wb{M}}

\newcommandx{\Dcmdp}[6][1=c,2=h,3=s,4=a]{D_{#2,#3,#4;#1}\prn*{#5,#6}}
\newcommandx{\Dcmdps}[6][1=c,2=h,3=s,4=a]{D_{#2,#3,#4;#1}^2\prn*{#5,#6}}

\newcommandx{\pim}[2][1=m,2=c]{\pi_{#1,#2}}
\newcommandx{\pimhat}[2][1=m,2=c]{\wh{\pi}\idx{#1,#2}}

\newcommandx{\muhat}{\wh{\mu}}
\newcommandx{\muhatm}[2][1=m,2=h]{\muhat\idx{#1}\ind{#2}}

\newcommandx{\thetahat}{\wh{\theta}}
\newcommandx{\thetahatm}[2][1=m,2=h]{\thetahat\idx{#1}\ind{#2}}

\newcommandx{\Gammam}[2][1=m,2=h]{\Gamma\idx{#1}\ind{#2}}
\newcommandx{\Sigmam}[2][1=m,2=h]{\Sigma\idx{#1}\ind{#2}}

\newcommand{\mainalgFull}{Layerwise pOLicy cover Inverse gaP weighting with trusted OccuPancy measures\xspace}
\newcommand{\mainalg}{LOLIPOP\xspace}

\newcommandx{\mask}[2][2=\veps]{\brk*{#1}_{#2}}

\newcommandx{\Stil}[2][1=m,2=h]{\wt{\cS}\idx{#1}\ind{#2}}
\newcommandx{\feat}{\phi}
\newcommandx{\featm}[2][1=m,2=h]{\phi\idx{#1}\ind{#2}}
\newcommandx{\chim}[2][1=m,2=h]{\chi\idx{#1}\ind{#2}}

\newcommandx{\chitil}[2][1=m,2=h]{\wt{\chi}\idx{#1}\ind{#2}}
\newcommandx{\ellm}[2][1=m,2=h]{\ell\idx{#1}\ind{#2}}
\newcommandx{\ellmsa}[4][1=m,2=h,3=s,4=a]{\ell\idx{#1,#3,#4}\ind{#2}}
\newcommandx{\Lambdatil}[2][1=m,2=h]{\wt{\Lambda}\idx{#1}\ind{#2}}
\newcommandx{\Lambdatilsa}[4][1=m,2=h,3=s,4=a]{\wt{\Lambda}\idx{#1,#3,#4}\ind{#2}}
\newcommandx{\Lambdam}[2][1=m,2=h]{\Lambda\idx{#1}\ind{#2}}
\newcommandx{\Lambdamsa}[4][1=m,2=h,3=s,4=a]{\Lambda\idx{#1,#3,#4}\ind{#2}}
\newcommandx{\Lambdahat}[2][1=m,2=h]{\wh{\Lambda}\idx{#1}\ind{#2}}
\newcommandx{\chihat}[2][1=m,2=h]{\wt{\chi}\idx{#1}\ind{#2}}

\newcommandx{\mustar}[1][1=h]{\mu\idx{\star}\ind{#1}}
\newcommandx{\muh}{\mustar}
\newcommand{\alg}{\mathsf{Alg}}

\input{style_macro/widebar}
\usepackage{color-edits}
 \addauthor{jq}{yellow!60!black}

\let\oldparagraph\paragraph
\renewcommand{\paragraph}[1]{\oldparagraph{#1.}}

\title{Offline Oracle-Efficient Learning for Contextual MDPs \looseness=-1\\via Layerwise Exploration-Exploitation Tradeoff}

\author{%
  Jian Qian\\
{\small\texttt{jianqian@mit.edu}}
\and
Haichen Hu\\
{\small\texttt{huhc@mit.edu}}
\and
David Simchi-Levi\\
{\small\texttt{dslevi@mit.edu}}
}
\date{}

\begin{document}

\maketitle

\begin{abstract}
    
Motivated by the recent discovery of a statistical and computational reduction from contextual bandits to offline regression \citep{simchi2020bypassing}, we address the general (stochastic) Contextual Markov Decision Process (CMDP) problem with horizon $H$ (as known as CMDP with $H$ layers). In this paper, we introduce a reduction from CMDPs to offline density estimation under the realizability assumption, i.e., a model class $\cM$ containing the true underlying CMDP is provided in advance. We develop an efficient, statistically near-optimal algorithm requiring only $O(H \log T)$ calls to an offline density estimation algorithm (or oracle) across all $T$ rounds of interaction. This number can be further reduced to $O(H \log \log T)$ if $T$ is known in advance. Our results mark the first efficient and near-optimal reduction from CMDPs to offline density estimation without imposing any structural assumptions on the model class. A notable feature of our algorithm is the design of a layerwise exploration-exploitation tradeoff tailored to address the layerwise structure of CMDPs. Additionally, our algorithm is versatile and applicable to pure exploration tasks in reward-free reinforcement learning.

\end{abstract}

\section{Introduction}
\label{sec:intro}

Markov Decision Processes (MDPs) model the long-term interaction between a learner and the environment and are used in diverse areas such as inventory management, recommendation systems, advertising, and healthcare \citep{puterman2014markov,sutton2018reinforcement}. The Contextual MDP (CMDP) extends MDPs by incorporating external factors, known as \emph{contexts}, such as gender, age, location, or device in customer interactions, or lab data and medical history in healthcare \citep{hallak2015contextual,modi2018markov}.
In an $H$-layer CMDP, the learner receives an instantaneous reward at each step over $H$ steps and aims to maximize the cumulative reward (return). For $T$ rounds of interaction, the learner's performance is measured by \emph{regret}, which is the difference between the total return obtained and that of an optimal policy.

In the special case of contextual bandits (one-layer CMDPs), a decade of research has led to algorithms with optimal regret bounds and efficient implementations with access to an offline regression algorithm (also termed as an \emph{offline regression oracle})~\citep{dudik2011efficient,agarwal2012contextual,agarwal2014taming,foster2018practical,foster2020beyond,simchi2020bypassing,xu2020upper}.
Most notably, \citet{simchi2020bypassing} demonstrates an offline-oracle-based algorithm FALCON that achieves optimal regret for general (stochastic) contextual bandits with access to an offline regression oracle (e.g., the Empirical Risk Minimization (ERM) oracle). Moreover, the algorithm is efficient given the output of the offline oracle (also referred to as offline oracle-efficient) and requires
only $O(\log \log T)$ calls to the oracle across all $T$ rounds if $T$ is known. 
These properties are highly desirable in practice since they reduce the computational problem of contextual bandits to the classical problem of offline regression with little overhead.  
However, to the best of our knowledge, no algorithm with these properties is available in the literature for general (stochastic) CMDPs. So, in this paper, we study the following question: 
\textit{Is there an offline-oracle-efficient 
algorithm 
that achieves the optimal regret 
for general (stochastic) CMDPs with only $O(H\log\log T)$ number of oracle calls?}

Several works have provided partial results for this question.
\citet{foster2021statistical} provides a general reduction from interactive decision making to online density estimation and has CMDP as an application. The proposed \textrm{E2D} algorithm achieves optimal regret but is online-oracle-efficient (as opposed to offline-oracle-efficient) since it requires access to an online density estimation algorithm.
\citet{foster2024online} provides a further reduction from online density estimation to offline density estimation, with the caveat that the reduction itself is inefficient. 
A separate thread of optimism-based algorithms for CMDPs extending the UCCB algorithm for contextual bandits \citep{xu2020upper} is studied by \citet{levy2023optimism,deng2024sample} with either assumption on the reachability of the CMDP or the representation structure of the CMDP (see \cref{app:related-works} for more details).
Last but not least, the algorithms proposed by \citet{foster2021statistical,foster2024online,levy2023optimism,deng2024sample} all require $O(T)$ number of oracle calls to the online/offline oracle respectively.

\begin{table}[!tp]
\caption{Algorithms' performance with general finite model class $\cM$ and i.i.d. contexts. The optimal rate here refers to $\wt{O}(\poly(H,S,A)\sqrt{T\log |\cM|})$.
All algorithms assume realizability, so it is omitted from the table. The reachability assumption and the varying representation assumption are very stringent for tabular CMDP, for details we refer to \cref{app:related-works}.}
\centering
\begin{tabular}{cccc}
\hline
Algorithm & Regret rate & Computational complexity                                                                             & Assumption             \\ \hline
E2D \cite{foster2021statistical}      & Optimal     & $O(T)$ calls to an online oracle                                                                     & No                     \\
RM-UCDD \cite{levy2023optimism}   & Suboptimal  & $O(T)$ calls to an offline oracle                                                                    & Reachability           \\
CMDP-VR \cite{deng2024sample}  & Optimal     & $O(T)$ calls to an offline oracle                                                                    & Varying Rep. \\
LOLIPOP (\textbf{this work})  & Optimal     & \begin{tabular}[c]{@{}c@{}}$O(\log T)$ or $O(\log\log T)$ \\ calls to an offline oracle\end{tabular} & No                     \\ \hline
\end{tabular}
\end{table}

In this work, we present an affirmative answer to the question by introducing the algorithm of \mainalg (\cref{alg:mainalg}). For $S$ number of states, $A$ number of actions, and a given model class $\cM$ where the underlying true model lies, the algorithm achieves the regret guarantee of 
$\wt{O}(\poly(H, S, A)\sqrt{T\log |\cM|})$.
This regret guarantee is minimax optimal up to $\poly(H, S, A)$ factor \citep{levy2023optimism}. The \mainalg algorithm assumes access to a Maximum Likelihood Estimation (MLE) oracle and is offline-oracle efficient. The results can be generalized to general offline density estimation oracles.
The most notable technical features are: (1) It generalizes the FALCON algorithm by \citet{simchi2020bypassing} to adapt to the multi-layer 
structure of a CMDP. 
More specifically, the FALCON algorithm is divided into $O(\log\log T)$ epochs, each corresponding to an oracle call, a fixed randomized policy. 
However, it is known for the MDPs that the learner has to switch its randomized policy at least $\wt{\Omega}(H)$ times to achieve sublinear regret \cite{zhang2022near}. 
Indeed, we further divide each epoch into $H$ segments, each with an oracle call, a new randomized policy for layerwise exploration-exploitation tradeoff. 
(2) In each segment, the exploration-exploitation tradeoff is done through Inverse Gap Weighting (IGW) on estimated regret for a set of explorative policies. 
The idea of running IGW on such a policy cover is proposed by \citet{foster2021statistical}. However, their policy cover is designed for $H$-layer exploration-exploration tradeoff and only works with strong online estimation oracles. In contrast, we refine the estimation of the occupancy measure layerwise by introducing the \emph{trusted occupancy measures}. This refinement enables our algorithm to work with offline oracles.
(3) Many other policy cover-based methods \citep{du2019provably,mhammedi2023representation,mhammedi2024efficient,amortila2024scalable} are developed for exploration tasks. Most notably, \citet{mhammedi2023representation} clips the occupancy measures on states with low reachability. Our approach takes a step forward to clip all transitions with low reachability to compute the trusted occupancy measures. \looseness=-1

Besides all the above novelties, the \mainalg algorithm is versatile and applicable to the pure exploration task of reward-free reinforcement learning for CMDPs. Concretely, it obtains near-optimal sample complexity of $O\prn*{  H^7S^4A^3\log (|\cM|/\delta)/\veps^2  }$ with only $O(H)$ number of oracle calls. Both the sample complexity bound and computational efficiency result for reward-free learning for stochastic CMDPs are new to the best of our knowledge.

\section{Preliminaries}
\label{sec:prelim}

\subsection{Problem Setup}
\label{sec:problem-setup}
\newcommand{\Mstar}{M_\star}
\newcommand{\OffDE}{\mathrm{OffDE}_\cM}
\newcommandx{\Pstar}[1][1=h]{P_\star\ind{#1}}
\newcommandx{\Rstar}[1][1=h]{R_\star\ind{#1}}
\newcommandx{\Qstar}[1][1=h]{Q_\star\ind{#1}}
\newcommandx{\Enmpic}[3][1=M,2=\pi,3=c]{\En^{#1,#2,#3}}
\newcommandx{\Pmpic}[3][1=M,2=\pi,3=c]{\bP^{#1,#2,#3}}
\newcommandx{\EnmpicR}[4][1=M,2=\pi,3=c,4=R]{\En^{#1,#2,#3,#4}}
\newcommandx{\PmpicR}[4][1=M,2=\pi,3=c,4=R]{\bP^{#1,#2,#3,#4}}
\newcommandx{\Pm}[1][1=M]{P_{#1}}
\newcommandx{\Rm}[1][1=M]{R_{#1}}

A Contextual Markovian Decision Process (CMDP) is defined by the tuple $(\mathcal{C}, M =\lbrace M(c)\rbrace_{c\in \cC},\mathcal{S},\mathcal{A},s^1)$, where $\cC$ is the contextual space, $\mathcal{S}$ is the state space, $\mathcal{A}$ is the action space and $s^1\in\cS$ is a fixed starting state independent of the context. We focus on tabular CMDPs which assumes $S=|\mathcal{S}|,A=|\mathcal{A}|< \infty$.
For any context $c\in\mathcal{C}$, $M(c)= \lbrace \Pm^{h}(c), \Rm^h(c)\rbrace_{h=1}^{H}$ consists of $H$-layers of probability transition kernel $\set{ \Pm^{h}(c)}_{h=1}^H$ and reward distributions $\set{ \Rm^{h}(c)}_{h=1}^H$, where $\Pm^h(c)$ and $\Rm^h(c)$ are specified by $\Pm^{h}(\cdot \mid{}s,a;c)\in\Delta(\cS)$ and $\Rm^h(s,a;c)\in\Delta([0,1])$ for all $h\in [H]$ and $s,a\in \cS\times\cA$.
For simplicity, we also denote $M =  \lbrace \Pm^{h}, \Rm^h\rbrace_{h=1}^{H}$, where $\Pm^{h} = \set{\Pm^{h}(c)}_{c\in \cC}$ and $\Rm^{h} = \set{\Rm^{h}(c)}_{c\in \cC}$.
Let $\PiRNS$ denote the set of all randomized, non-stationary policies, where any $\pi= (\pi\ind{1},\dots,\pi\ind{H}) \in \PiRNS$ has $\pi^h:\cS\to \Delta(\cA)$ for any $h\in [H]$.  
We use $T$ to denote the total number of rounds and $H$ to denote the horizon (the total number of layers). Let $\Mstar = \lbrace\Pstar, \Rstar\rbrace_{h\in [H]}$ be the underlying true CMDP the learner interacts with.
The interactive protocal proceeds in $T$ rounds, where for each round $t$, the $t$-th trajectory is generated as:
\vspace{-5pt}
\begin{enumerate}[label=$\bullet$,leftmargin=5mm]
    \setlength{\itemsep}{0pt}
    \setlength{\parsep}{0pt}
    \setlength{\parskip}{0pt}
    \item A context $c_t$ is draw i.i.d. from an unknown distribution $\cD$ and $s_t^1 = s^1$.
    \item The learner chooses the policy $\pi_t$ based on the context.
    \item For $h=1,\dots,H$:
    \begin{enumerate}[label=$\bullet$,leftmargin=5mm]
        \setlength{\itemsep}{0pt}
        \setlength{\parsep}{0pt}
        \setlength{\parskip}{0pt}
        \item The action is drawn from the randomized policy $a\idx{t}\ind{h}\sim \pi_t\ind{h}(s\idx{t}\ind{h})$.
        \item The reward and the next state is drawn respectively from the reward distribution and the transition kernel, i.e., $r\idx{t}\ind{h}\sim \Rstar(s\idx{t}\ind{h},a\idx{t}\ind{h};c_t)$ and $s\idx{t}^{h+1}\sim \Pstar(\cdot\mid{}s\idx{t}\ind{h},a\idx{t}\ind{h};c\idx{t})$.
    \end{enumerate}
\end{enumerate}
\vspace{-5pt}
Without lose of generality, throughout the paper, we assume that
the total reward $0\le\sum_{h=1}^{H}r^h\le 1$ almost surely.
For any model $M$, context $c$ and policy $\pi$, we use $M(\pi,c)$ to denote the distribution of the trajectory $c_1, \pi_1, s_1\ind{1}, a_1\ind{1}, r_1\ind{1}, \dots, s_1\ind{H}, a_1\ind{H}, r_1\ind{H}$ given $\Mstar = M$, $c_1=c$, and $\pi_1=\pi$. Also denote the probability and the expectation under $M(\pi,c)$ to be $\Pmpic\prn{\cdot}$ and $\Enmpic\brk{\cdot}$ respectively.
Given any policy $\pi$, state $s$ and action $a$, we define the action value function $\Qstar(s,a;\pi,c)$ at layer $h$ and the value function $\Vstar[h](s;\pi,c)$ at layer $h$ under context $c$ and policy $\pi$ as
\begin{align*}
    \textstyle \Qstar(s,a;\pi,c)=\sum\nolimits_{j=h}^{H}\Enmpic[\Mstar] \brk{r_1^{j} \mid{} s_1\ind{h},a_1\ind{h} = s,a } 
    \text{~~and~~}\Vstar[h](s;\pi,c)= \max_{a\in \cA} \Qstar(s,a;\pi,c).
\end{align*}
We denote the optimal policy under context $c$ as $\pistar$ and abbreviate its value function as $\Vstar[h](\cdot;c)$. 
For $h=1$, we further simply the notation by denoting $\Vstar(c) = \Vstar(s^1;c)$ and $\Vstar(\pi,c) = \Vstar(s^1;\pi,c)$.
The regret of policy $\pi$ under context $c$ and the total regret\footnote{The regret we defined here is conventionally known as the pseudo-regret in the literature. The conventional regret defined as $\sum\nolimits_{t=1}^{T}\reg(\pi_t,c_t)$ can be bounded by the pseudo-regret up to an additional $O(\sqrt{T\log(1/\delta)})$ term with a standard concentration argument, which we omit here for simplicity.} of the learner are defined as 
\begin{align*}
        \textstyle  \reg(\pi,c)=\Vstar(c)-\Vstar(\pi,c) \quad\text{and}\quad \Reg(T)\ldef \sum\nolimits_{t=1}^{T} \En_{t}[\reg(\pi_t,c_t)], 
\end{align*}
where $\En_t[\cdot]$ is the conditional expectation given the interaction up to round $t$.

\begin{assumption}[Realizability]
The learner is given a model class $\cM$ where the true underlying model $\Mstar$ lies, that is, $\Mstar\in \cM$.
\end{assumption}

\subsection{Offline Density Estimation Oracles}

For any model class $\cM$, a general offline density estimation oracle associated with $\cM$, denoted by $\OffDE$, is defined as an algorithm that generates a predictor $\Mhat$ based on the input data and $\cM$. 
In this paper, we measure the performance of the predictor in terms of the squared Hellinger distance, which is defined for any two distributions $\bP$ and $\bQ$ 
for any common dominating measure $\nu$\footnote{The value is independent of the choice of $\nu$.} by
$$
\textstyle
\Dhels{\bP}{\bQ} \ldef \frac{1}{2} \int \prn*{ \sqrt{\frac{d \bP}{d\nu}} - \sqrt{\frac{d \bQ}{d \nu}} }^2 d \nu,
$$
For our purpose, we are interested in the following statistical guarantee.
\begin{definition}[Offline density estimation oracle]
    \label{def:offline-DE-oracle}
Let $p$ be a map from a context to a distribution on the set of policies $\PiRNS$, that is, for any $c\in \cC$, $p(c)\in \Delta(\PiRNS)$.  Given $n$ training trajectories $(c_i,\pi_i, s_i\ind{1}, a_i\ind{1}, r_i\ind{1}, \dots, s_i\ind{H}, a_i\ind{H}, r_i\ind{H})$ i.i.d. drawn according to $ c_i \sim \cD$, $\pi_i\sim p(c_i)$ and $s_i\ind{1}, a_i\ind{1}, r_i\ind{1}, \dots, s_i\ind{H}, a_i\ind{H}, r_i\ind{H}$ be the trajectory sampled according to $\Mstar(\pi_i,c_i)$. The offline density estimation oracle $\OffDE$ returns a predictor $\Mhat$. For any $\delta\in (0,1/2)$, with probability at least $1-\delta$, we have
$$
    \textstyle 
\En_{c\sim \cD,\pi\sim p(c)}\brk*{\Dhels{\Mhat(\pi,c)}{\Mstar(\pi,c)} } \leq \cE_{\cM,\delta}(n).
$$
\end{definition}
\newcommand{\MLE}{\mathsf{MLE}_{\cM}}
The Maximum Likelihood Estimation estimator $\MLE$ is an example of an offline density estimation oracle that achieves $\cE_{\cM,\delta}(n) \lesssim \log (|\cM|/\delta)/n$ (see more details in \cref{app:technical}).

\paragraph{Addtional Notation}
For any integer $n$, we use $[n]$ to denote the set $\set{1,\dots,n}$. For any set $\cX$, we use $\Delta(\cX)$ to denote the set of all distributions on the set $\cX$. We define $O(\cdot)$, $\Omega(\cdot)$, $\Theta(\cdot)$, $\wt{O}(\cdot)$, $\wt{\Omega}(\cdot)$, $\wt{\Theta}(\cdot)$ following standard non-asymptotic big-oh notation. We use the binary relation $x \lesssim y$ to indicate that $x\leq O(y)$. $\mathbbm{1}(\cE)$ is an indicator function of event $\cE$.

\section{Related Works}
\label{app:related-works}
\paragraph{Contextual bandits and contextual MDPs} 
The SquareCB \citep{foster2020beyond} obtains optimal regret for contextual bandits with access to an online regression oracle. This is extended to the CMDPs by \citet{foster2021statistical} with the E2D algorithm. However, the algorithm requires $O(T)$ called to an online density estimation oracle. Compared to our algorithm, it necessitates significantly more calls to an oracle that is harder to implement for a general model class $\cM$.
After the FALCON algorithm \citep{simchi2020bypassing} establishes the reduction from contextual bandits to offline regression, \citet{xu2024upper} proposed the UCCB algorithm, which is less oracle-efficient in terms of oracle calls but adopts the prevalent "optimism in the face of uncertainty" principle and is thus easier to generalize. More specifically, the UCCB algorithm is extended to CMDPs by \citet{levy2023optimism,deng2024sample} with assumptions on the model class. The RM-UCDD algorithm proposed by \citet{levy2023optimism} requires the model class to have a minimum reachability $p_{\min}$ to all states under any policy and the regret guarantees scale with $O(\poly(H,S,A)\cdot (1/p_{\min}) \sqrt{T\log |\cM|})$. This assumption precludes model classes with small reachability, which frequently happens in practice \citep{agarwal2020flambe}. The CMDP-VR algorithm proposed by \citet{deng2024sample} assumes a varying representation assumption on the model class instead. The assumption asserts that any model $M = \set{P^h,R^h}_{h\in [H]} \in \cM$ satisfies for any $c,h,s,a,s'$, $P^h(s'|s,a;c) = \inner{\phi^h(s,a;c), \mu^h(s')}$ for the known feature vector $\phi^h(s,a;c)\in \bR^d$ and an unknown vector $\mu^h(s')\in \bR^d$ which does not depend on the context $c$. This assumption is stringent because canonically, the feature vector for CMDPs will be chosen to be the unit vector indexed by $s,a$ in $\bR^{SA}$, i.e., $\phi^h(s,a;c) = e_{s,a} \in \bR^{SA}$. Then, the requirement of $\mu^h(s')$ not depending on the context forces $P^h$ to not depend on the context as well. This reduces the CMDP to an MDP. While it is possible to complicate the feature vector to not reduce to an MDP, this would result in a higher dimension in the feature vector space. The increase of the feature dimension will be reflected in the regret bounds obtained $\wt{O}(\poly(H,d)\sqrt{T\log |\cM|} )$. Another significant disadvantage compared with our algorithm is that the RM-UCDD and CMDP-VR algorithm requires $O(T)$ number of oracle calls.\looseness=-1

\paragraph{Reward-free reinforcement learning} Reward-free reinforcement learning aims to efficiently explore the environment without relying on observed rewards. By doing so, it aims to enable the computation of a nearly optimal policy for any given reward function, utilizing only the trajectory data collected during exploration and without needing further interaction with the environment. This framework is proposed by \cite{jin2020reward} and has been extensively studied for MDPs with various assumptions \citep{zhang2020nearly,agarwal2020flambe,wang2020reward,zhang2022efficient,chen2022statistical,miryoosefi2022simple,wagenmaker2022regret,hu2022towards,cheng2023improved,li2023optimal,mhammedi2024efficient,amortila2024scalable}. However, to the best of our knowledge, we are the first to study the reward-free reinforcement learning setting for stochastic CMDPs. We provide a near-optimal sample complexity upper bound and a matching lower bound up to a $\poly(H,S,A)$ factor with only $O(H)$ number of oracle calls. Nevertheless, the upper bound is obtained by adjusting the exploration-exploitation, highlighting the flexibility of our algorithm.

\section{Main Results and Algorithm}
\label{sec:algorithm}
\newcommand{\Ber}{\mathrm{Ber}}
\newcommand{\pihat}{\wh{\pi}}
In this section, we present our main results and introduce the algorithm of \mainalg (\cref{alg:mainalg}). First, we give an overview of the algorithm. Then, we discuss the theoretical guarantees obtained by this algorithm. Finally, we introduce the algorithm's different components with corresponding guarantees. All proofs are deferred to \cref{app:algorithm}.

\subsection{Main Results}

\paragraph{Overview of \cref{alg:mainalg}} The algorithm proceeds with epochs. The total number of $T$ rounds is divided into $N$ epochs. For an epoch schedule $0=\tau_0<\tau_1<\dots<\tau_N =T/H$ to be specified later, the $m$-th epoch will last $H(\tau_m-\tau_{m-1})$ rounds. Furthermore, each epoch is evenly divided into $H$ segments, each consisting of $\tau_m - \tau_{m-1}$ rounds. During the $h$-th segment in $m$-th epoch, a kernel $\metapol:\cC\to \Delta(\Pi)$ will be specified to determine the policy. More specifically, upon receiving the context $c_t$, a policy $\pi_t$ will be sampled from $\metapol(c_t)$ and executed. After collecting the trajectories $\set{c_t, \pi_t}\cup\set{s_t\ind{j}, a_t\ind{j}, r_t\ind{j}}_{j\in [H]}$ in the $h$-th segment of the $m$-th epoch for 
$\tau_{m-1}H+(\tau_{m}-\tau_{m-1})(h-1)+1\leq t\leq \tau_{m-1}H+(\tau_{m}-\tau_{m-1})h$,
the offline density estimation oracle $\OffDE$ is called with these trajectories as input. Denote the output  $\Mhat\idx{m}\ind{h}$, we will only be interested in the $h$-th layer of this output, which we denote by $\lbrace\widehat{P}_{m}^{h},\widehat{R}_{m}^h\rbrace$. Then the collections of estimators $ \Mhat\idx{m} = \lbrace\widehat{P}_{m}^{h},\widehat{R}_{m}^{h}\rbrace_{h=1}^{H}$ will be used for the next epoch.

Throughout this paper, we will adopt the following convention for the free variables $m,\pi,c,h,s,a$. They will be used to denote an epoch index in $[N]$, a policy in $\PiRNS$,  a context in $\cC$,  a layer index in $[H]$,  a state in $\cS$, and an action in $\cA$ respectively.

Before we dive into the details of the algorithm, we highlight first the theoretical guarantees obtained.
\begin{theorem}
        \label{thm:loglogT}
If $T$ is known, then by choosing the epoch schedule $\tau_m = 2(T/H)^{1-2^{-m}}$ for $m\geq 1$ and the offline density estimation oracle $\OffDE = \MLE$, 
the outputs $\set{\pi_t}_{t\in [T]}$ of \cref{alg:mainalg}
    satisfies that with probability at least $1-\delta$, 
    \begin{align*}
        \Reg(T)\lesssim\sqrt{ H^7S^4A^3 T\cdot \log(|\cM|\log \log T/\delta)\log\log T}
    \end{align*}
with only $O(H\log\log T)$ number of oracle calls to the $\MLE$ oracle for $\delta\in (0,1/2)$.
\end{theorem}

\begin{theorem}
        \label{thm:logT}
If $T$ is not known, then by choosing the epoch schedule $\tau_m = 2^m$ for $m\geq 1$ and the offline density estimation oracle $\OffDE = \MLE$,
the outputs $\set{\pi_t}_{t\in [T]}$ of \cref{alg:mainalg}
    satisfies that with probability at least $1-\delta$, 
    \begin{align*}
        \Reg(T)\lesssim\sqrt{ H^7S^4A^3 T\cdot \log(|\cM| \log T/\delta)}
    \end{align*}
with $O(H\log T)$ number of oracle calls to the $\MLE$ oracle for $\delta\in (0,1/2)$.
\end{theorem}

The theorems above show that \cref{alg:mainalg} with both epoch schedules achieve near-optimal statistical complexity that a matches the lower bound of $\Omega(\sqrt{HSAT\log |\cM|/\log A})$ proven by \citet{levy2023optimism} up to a $\poly(H,S,A)$ factor.

\paragraph{Computational efficiency} Consider the epoch schedule $\tau_m = 2^{m}$ for $m \in \mathbb{N}$ as discussed in \cref{thm:logT}. For any unknown $T$, our algorithm operates over $O(\log T)$ epochs, making one oracle call per epoch. Thus, the computational complexity is $O(\log T)$ oracle calls over $T$ rounds, with an additional per-round cost of $O(\poly(H, S, A, \log T))$. This offers potential advantages over existing algorithms that achieve near-optimal rates without assumptions beyond realizability.
The E2D algorithm \citep{foster2021statistical}, for instance, requires $O(T)$ calls to an online density estimation oracle, involving significantly more calls to a more complex oracle for a general model class $\mathcal{M}$. On the other hand, the Version Space Averaging + E2D algorithm \citep{foster2024online} requires $O(T)$ calls to an offline density estimation oracle and incurs a computational cost scaling with $O(|\mathcal{M}|)$ per round. Compared to our algorithm, this results in far more oracle calls and considerably higher computational costs per round.\looseness=-1

If the total number of rounds $T$ is known to the learner, we can further reduce the computational cost of \mainalg. For any $T \in \mathbb{N}$, consider the epoch schedule $\tau_m = 2(T/H)^{1-2^{-m}}$ as in \cref{thm:loglogT}, similar to \citet{simchi2020bypassing}. In this scenario, \mainalg will run in $O(\log \log T)$ epochs, making only $O(\log \log T)$ oracle calls over $T$ rounds while still maintaining a slightly worse regret guarantee. \looseness=-1

\begin{algorithm}[t]
\caption{\mainalgFull (\mainalg)}
\label{alg:mainalg}
\begin{algorithmic}[1]
    \Require epoch schedule $0=\tau_{-1}=\tau_0<\tau_1<\dots<\tau_N =T/H$, confidence parameter $\delta\in (0,1/2)$, 
    model class $\cM$, offline oracle $\OffDE$. 
\State Initialize: $\Mhat_0 = \lbrace\widehat{P}^h_{0}, \widehat{R}_0\rbrace_{h=1}^{H}$, where $\widehat{P}^h_{0}$ is any transtion kernel and $\widehat{R}_0$ is constantly $0$.
\For{epoch $m=1,2,\cdots,N$}
\State \multiline{Let $\cE_m = \cE_{\cM,\nicefrac{\delta}{2N^2}}(\tau_{m-1}-\tau_{m-2})$, 
$\gammam =\sqrt{\nicefrac{H^6S^4A^3}{\cE_m}}$ and  
$\etam=\nicefrac{\gammam}{720eH^5S^3A^2}$.}
\For{segment $h=1,\dots,H$}
\For{round $t=\tau_{m-1}H+(\tau_{m}-\tau_{m-1})(j-1)+1,\cdots,\tau_{m-1}H+(\tau_{m}-\tau_{m-1})h$}
\State Observe context $c_t\in\mathcal{C}$ from the environment. 
\For{$s,a\in \cS\times \cA$} 
\State \multiline{Compute 
\begin{align*}
    \pi^{h,s,a}_{m,c_t}=\argmax_{\pi}\frac{\widetilde{d}^{h}_{m}(s,a;\pi,c_t)}{SA+\etam\cdot \reghatm[m-1] (\pi,c_t)},
\end{align*}
where $\ocmtil$ is the trusted occupancy measure defined as in \cref{trusted-occu}. \label{line:policy-covering}}
\EndFor
\State  Let the policy cover $\Pi^{h}_{m}(c_t)=
\set{ \widehat{\pi}_{m-1,c_t}} \cup \lbrace \pi^{h,s,a}_{m,c_t}\rbrace_{s,a\in\mathcal{S}\times\mathcal{A}}$. \label{line:optimal-policy} 
\State \multiline{Define $\metapol(c_t)$ to be the Inverse Gap Weighting distribution on the policy cover $\Pi^{h}_{m,c_t}$ 
\begin{align}
    p_{m}^{h}(c_t,\pi)=\frac{1}{\lambda_{m,c_t}^{h}+\etam\cdot\reghatm[m-1](\pi,c_t)},\quad \forall \pi\in\polCov(c_t), \label{def:igw}
\end{align}
where $\lambda_{m,c_t}^{h}$ is the constant that normalize the distribution.} \label{line:igw}
\State \multiline{Sample and execute $\pi_t\sim p_{m}^{h}(c_t)$ and observe the trajectory $c_t,\pi_t,\set{s_t^j,a_t^j,r_t^j}_{j\in [H]}$.}
\EndFor
\State \multiline{Run $\OffDE$ with the input trajectories $\set{c_t,\pi_t,\set{s_t^j,a_t^j,r_t^j}_{j\in [H]}}_{t:m(t) = m}$ 
and obtain the $h$-th layer estimator $\widehat{P}^{h}_{m}$ and $\widehat{R}_{m}^{h}$. \label{line:model-estimation}}
\EndFor
\EndFor
\end{algorithmic}
\end{algorithm}

\subsection{Detailed Construction and Guarantees of Each Component}
In this section, we explain in detail our construction and the guarantees of each component along the dependence graph (\pref{fig:process}). We first introduce how the estimators $\set{\Phatm[m-1],\Rhatm[m-1]}_{h\in [H]}$ from the previous epoch are used in the new epoch. Then we proceed to introduce how to construct $\metapol(c_t)$ given $\polCov[m][h](c_t)$ 
. Next, we introduce how are $\Phatm,\Rhatm$ obtained given $\metapol$. Subsequently, we present the definition of the set of trusted transitions $\Ttil$ and trusted occupancy measure $\ocmtil$. Finally, we present how the policy cover $\polCov[m][h](c_t)$ is constructed.

During epoch $m$, we will be using the estimators $\lbrace\widehat{P}_{m-1}^{h},\widehat{R}_{m-1}^{h}\rbrace_{h=1}^{H}$ from the previous epoch for regret estimation. More specifically, for $\pi,c,h,s$, we define the value functions with respect to the model $\lbrace\widehat{P}_{m-1}^{h}(c),\widehat{R}_{m-1}^{h}(c)\rbrace_{h=1}^{H}$ as $\Vhatm[m-1][h](s;\pi,c)$. 
The optimal value function is denoted by $\Vhatm[m-1](s;c) = \max_\pi \Vhatm[m-1](s;\pi,c)$.  For $h=1$, we further simply the notation by denoting $\Vhatm[m-1](c) = \Vhatm[m-1](s^1;c)$ and $\Vhatm[m-1](\pi,c) = \Vhatm[m-1](s^1;\pi,c)$.
Also denote the optimal policy under context $c$ by $\widehat{\pi}_{m-1,c} = \argmax_{\pi} \widehat{V}^{1}_{m-1}(\pi,c)$. 
Thus, the regret is estimated to be
\begin{align*}
    \textstyle     \reghatm[m-1](\pi,c)=\widehat{V}^{1}_{m-1}(c)-\widehat{V}^{1}_{m-1}(\pi,c).
\end{align*}

\begin{figure}[t]
    \centering
\adjustbox{scale=.88, center}{
\begin{tikzcd}
	{{\left\{ \widehat{P}_{m-1}^h, \widehat{R}_{m-1}^h \right\}_{h\in [H]}}} & {\Pi_m^1, p_m^1} && {\Pi_m^2, p_m^2} && {......} \\
	& {\widehat{P}_m^1, \widehat{R}_m^1} & {\widetilde{\mathcal{T}}_m^1,\widetilde{d}_m^2} & {\widehat{P}_m^2, \widehat{R}_m^2} & {\widetilde{\mathcal{T}}_m^2,\widetilde{d}_m^3} & {\widehat{P}_m^{H}, \widehat{R}_m^{H}}
	\arrow[from=1-1, to=1-2]
 	\arrow[bend left = 10, from=1-1, to=1-4]
	\arrow[bend left = 13, from=1-1, to=1-6]
	\arrow[from=1-2, to=2-2]
	\arrow[from=1-4, to=2-4]
	\arrow[from=1-6, to=2-6]
	\arrow[from=2-2, to=2-3]
	\arrow[from=2-3, to=1-4]
	\arrow[from=2-4, to=2-5]
	\arrow[from=2-5, to=1-6]
\end{tikzcd}
}
\caption{The dependence graph of the construction. The estimation $\Mhat_{m-1}=\set{ \widehat{P}_{m-1}^h, \widehat{R}_{m-1}^h }_{h\in [H]}$ from the previous round provides the optimal policy $\pihat_{m-1}$ (\pref{line:optimal-policy}) and the regret estimation $\reghatm[m-1]$ (\pref{line:policy-covering}) for the construction of $\Pi_{m}^h, p_m^h$. The estimation $\widehat{P}_m^h, \widehat{R}_m^h$ is generated by calling the oracle $\OffDE$ on the trajectories collected with policy kernel $\metapol$ (\pref{line:model-estimation}). The trusted transitions and trusted occupancy measures $\widetilde{\mathcal{T}}_m^h,\widetilde{d}_m^{h+1}$ are computed from $\widetilde{d}_m^{h}, \widehat{P}_m^h$ (\cref{def:trusted-transtion,def:trusted-occupancy-measure}). The policy cover $\Pi_{m}^h$ is the union of $\pimhat[m-1][\cdot]$ and the policies $\set{\pi^{h,s,a}_{m,\cdot}}_{s,a\in \cS\times\cA}$ calculated in \pref{line:policy-covering} which requires $\widetilde{\mathcal{T}}_m^{h-1},\widetilde{d}_m^{h}$. The policy kernel $p_m^h$ is the inverse gap weighting on $\Pi_{m}^h$ (\pref{line:igw}).}
\label{fig:process}
\end{figure}

At round $t$, let $m(t)$ and $h(t)$ be the epoch in which the segment round $t$ lies. 
We note that during each epoch $m$ and segment $h$, all of the notions $\polCov[m][h](c_t),\widehat{P}^{h}_{m}(c_t)$, $\widehat{R}^{h}_{m}(c_t)$, $\Ttil(c_t)$, $\ocmtil(\cdot,\cdot;\cdot,c_t)$, $p_m^h(c_t)$, and $\pi^{h,s,a}_{m,c_t}$ will not depend on the specific time step $t$, but only the context $c_t$.
Thus, we will use $\polCov[m][h](c)$ to denote the policy cover if $c_t=c$ when $m(t),h(t) = m,h$. Similar conventions regarding the context $c$ apply to the notations $\widehat{P}^{h}_{m},\widehat{R}^{h}_{m},\Ttil, \ocmtil , p_m^h, \pi^{h,s,a}_{m,\cdot}$.
Under any context $c$, the policy cover $\polCov$ will include $\pihat_{m-1,c}$ and has no more than $SA+1$ policies. These two properties together guarantee that the Inverse Gap Weighting \citep{foster2020beyond} randomized policy $\metapol(c)$ (\pref{line:igw}) satisfies the following guarantee on the estimated regret.

\begin{lemma}
\label{lem:igw-regret-bound}
    For any $m$, $h$, $c$, the definition of the randomized policy $\metapol(c)$ is well defined, i.e., there exist $\lambda_{m,c}^{h}\in [0,2SA]$ such that $\sum_{\pi\in \Pi^{h}_{m}(c)}p_{m,c}^{h}(\pi)=1$. Furthermore, we have the estimated regret is bounded by $ \En_{\pi\sim \metapol(c)}\brk*{\reghatm[m-1](\pi,c)} \lesssim  \sqrt{H^{4}S^{4}A^{3}\cdot \cE_m  }.$
\end{lemma}
The choice of $\lambdam$ here is for compactness of presentation. It can be chosen to be $2SA$ for suboptimal arms and collect the probability remained to the optimal arm \citep{simchi2020bypassing}, which is computationally efficient. The computation for the policy $\pi^{h(t),s,a}_{m(t),c_t}$ for any $t,s,a$ can be computed in $\poly(H, S, A,\log T)$ time by formulating it as a linear fractional programming problem. We defer the details to \cref{app:compute}. 

Since $\metapol$ maps $\cC$ to randomized policies, it is thus a policy kernel. This means the trajectories generated in epoch $m$ and segment $h$ follow an i.i.d. distribution as described in the definition of \cref{def:offline-DE-oracle}. By applying the guarantee from \cref{def:offline-DE-oracle}, we have the following guarantee on $\widehat{P}^{h}_{m},\widehat{R}^{h}_{m}$. 
\looseness=-1
\begin{lemma}
    \label{lem:offDE-bound}
    For any $m$, $h$, and $c\sim \cD,\pi\sim \metapol(c)$, we have with probability at least $1-\frac{\delta}{2N^2} $, that 
    \begin{align*}
        \bE_{c,\pi}
            \left[\Enmpic[\Mstar] \left[ 
            \Dhels{\widehat{P}^{h}_{m}(s_1^h,a_1^h;c)}{\Pstar(s_1^h,a_1^h;c)}+\Dhels{\widehat{R}_{m}^{h}(s_1^h,a_1^h;c)}{\Rstar(s_1^h,a_1^h;c)}
        \right]\right]
        \le  \cE_m
    \end{align*}
    for $\delta\in (0,1/2)$.
\end{lemma}
If the offline density estimation oracle is chosen to be the Maximum Likelihood Estimation oracle $\MLE$, we will obtain $\cE_m \lesssim \log(T\cM/\delta)/(\tau_{m-1}-\tau_{m-2})$.

The most involved part of our construction concerns the idea of \emph{trusted transitions} and \emph{trusted occupancy measures}. This construction eliminates the parts of transitions that are too scarcely visited. The purpose will be clear in the guarantees (\cref{lem:dropped-state-bound,lem:bound_P_and_d}) subsequent to the definitions.
\begin{definition}
\label{trusted-occu}
    For any $m,\pi,c,h,s,a$, we define iteratively the \textbf{trusted occupancy measures} $\widetilde{d}^h_{m}(s;\pi,c)$, $\widetilde{d}^{h}_{m}(s,a;\pi,c)$ and the set of \textbf{trusted transitions} $\widetilde{\mathcal{T}}_{m}^{h}(c)$ at layer $h$ as the following:
    \begin{align}
        \label{def:trusted-occupancy-measure}
        \begin{split}
            \widetilde{d}^1_{m}(s;\pi,c)&=\mathbbm{1}(s=s^1), \quad\quad \widetilde{d}^1_{m}(s,a;\pi,c)=\mathbbm{1}(s=s^1)\pi^1(s,a),\\
            \widetilde{d}^{h}_{m}(s;\pi,c)&\ldef\sum\nolimits_{s',a',s\in\widetilde{\mathcal{T}}_m^{h-1}(c)}\widetilde{d}^{h-1}_{m}(s',a';\pi,c)\widehat{P}^{h-1}_{m}(s|s',a';c),\\
            \widetilde{d}^{h}_{m}(s,a;\pi,c)&\ldef\widetilde{d}^{h}_{m}(s;\pi,c)\pi^{h}(s,a).
        \end{split}
    \end{align}
    For any $m,h,c$, the set of trusted transitions 
    are defined as the set of transitions 
    \begin{align}
        \label{def:trusted-transtion}
        \widetilde{\mathcal{T}}_m^h(c)\triangleq\left\lbrace(s,a,s')\Big|\max_{\pi}\frac{\widetilde{d}^h_{m}(s,a;\pi,c)}{SA+\etam\cdot \reghat_{m-1}(\pi,c)}\cdot\widehat{P}^{h}_{m}(s'|s,a;c)\ge \frac{1}{\zetam}\right\rbrace,
    \end{align}
    where $\zetam=\frac{\gammam}{8eH(H+1)^2}$. 
    Notice that to define $\widetilde{d}^{h}_{m}(s;\pi,c)$ and $\widetilde{d}^{h}_{m}(s,a;\pi,c)$, we \textbf{only} need $\set{\widetilde{\mathcal{T}}_m^{j}(c), \widehat{P}^{j}_{m}(c)}_{j\in [h-1]}$. Thus, the two definitions are iteratively well-defined.
    Meanwhile, we also define the \textbf{observable occupancy measures} as the occupancy measures of the true model going through only the trusted transitions, i.e.,
    \begin{align*}
    \begin{split}
        d^1_{m}(s;\pi,c)&=\mathbbm{1}(s=s^1), \quad\quad d^1_{m}(s,a;\pi,c)=\mathbbm{1}(s=s^1)\pi^1(s,a),\\
        d^{h}_{m}(s;\pi,c)&\ldef\sum\nolimits_{s',a',s\in\widetilde{\mathcal{T}}_m^{h-1}(c)}d^{h-1}_{m}(s',a';\pi,c)\Pstar[h-1](s|s',a';c),\\
        d^{h}_{m}(s,a;\pi,c)&\ldef d^{h}_{m}(s;\pi,c)\pi_{h}(s,a).
    \end{split}
    \end{align*} 
\end{definition}

The computation of the set of trusted transitions need not enumerate all policies. The trusted transition set can be computed in $\poly(H, S, A,\log T)$ time by formulating it as a linear fractional programming problem. We defer the details to \cref{app:compute}.

Define the estimated occupancy measures $\widehat{d}_{m}^{h}(s;\pi,c) \ldef \Enmpic[\Mhat_m] \brk{\mathbbm{1}(s_1^h= s)}$ and $\widehat{d}_{m}^{h}(s,a;\pi,c) \ldef \Enmpic[\Mhat_m] \brk{\mathbbm{1}(s_1^h,a_1^h= s,a)}$. The trusted occupancy measure, though it eliminates rarely visited transitions, remains a valid estimate for all policies because the divergence between the estimated occupancy measure and itself is bounded. Specifically, we have the following lemma.
\begin{lemma}
\label{lem:dropped-state-bound}
    For all $m, \pi, h,s,a$, under any context $c$, we have
\begin{align*}
       \textstyle  \widehat{d}_{m}^{h}(s,a;\pi,c) - \widetilde{d}_{m}^{h}(s,a;\pi,c)\leq  32e \sqrt{H^4S^2A\cdot \cE_m}  + \widehat{\mathrm{reg}}_{m-1}(\pi;c)/(90HSA).
\end{align*}
\end{lemma}
The next guarantee is the key to our analysis and is the most non-trivial guarantee of our construction. The following lemma states that if, for a context $c$, the Hellinger divergence between $\Phat$ and $\Pstar$ at layer $h$ is small for all $h\in [H]$, 
then the trusted occupancy measure is upper bounded by a scaling of the observable occupancy measure.
\begin{lemma}
    \label{lem:bound_P_and_d}
    For any $m$ and $c$,
     assume for all $h\in[H]$,
    $$
        \textstyle \bE_{\pi\sim  \metapol(c)}\Enmpic[\Mstar]\left[
        \Dhels{\widehat{P}^h_{m}(s_1^h,a_1^h;c)}{\Pstar(s_1^h,a_1^h;c)}
        \right]\le H/\gamma_m.$$
Then for the same $m,c$ and all $\pi,h,s,a$, we have
\begin{align*}
       \textstyle  \widetilde{d}^{h}_{m}(s,a;\pi,c)\le \prn*{1+1/H}^{2(h-1)}d^h_{m}(s,a;\pi,c).
\end{align*}
\end{lemma}
Since the trusted occupancy measure is upper bounded up to scaling by the observable occupancy measure, then the state-action pairs with large trusted occupancy measures are guaranteed to be visited often in the true dynamics as well. 
This enables more accurate planning and is thus crucial to our analysis. \looseness=-1

Finally, we state the coverage guarantee achieved by the construction of $\polCov[m][h]$ and $\metapol$. Concretely, we upper bound the trusted occupancy measure $\ocmtil(\cdot,\cdot;\pi,\cdot)$ of any policy $\pi$ by the trusted occupancy measure induced by policy kernel $\metapol$.

\begin{lemma}
    \label{lem:coverage-bound}
    For any $m,\pi,c,h,s,a$, 
    we have
    \begin{align}
        \begin{split}
            \revindent[.7]\ocmtil (s,a;\pi,c)\cdot D_{\mathrm{H}}\prn{\Pstar(s,a;c),\Phatm(s,a;c)}
            \\
            &\leq \frac{\gammam}{e^2H}\cdot\metapol(c,\pol) \ocmtil(s,a;\pol,c)\cdot D_{\mathrm{H}}^2\prn{\Pstar(s,a;c),\Phatm(s,a;c)} + \cE_m', \label{ineq:coverage}
        \end{split}
    \end{align}
    where $\cE_m' =    \prn*{2e^2\sqrt{ \frac{\cE_m}{H^4S^2A}}  + \frac{1}{720 H^4S^3A^2}\reghatm[m-1](\pi,c) }\ocmtil (s,a;\pi,c)$  ,and $\metapol(c,\pol)$ is the probability of $\metapol(c)$ on $\pol$. The guarantee \cref{ineq:coverage} also holds replacing $D_{\mathrm{H}}\prn{\Pstar(s,a;c),\Phatm(s,a;c)}$ with $D_{\mathrm{H}}\prn{\Rstar(s,a;c),\Rhatm(s,a;c)}$ on both sides of the inequality.
\end{lemma}

\section{Regret Analysis}
\label{sec:regret-analysis}
In this section, we prodive a proof sketch of the regret analysis. Detailed proofs are deferred to \cref{app:regret-analysis}. 
We first aggregate the component-wise guarantees (\cref{lem:igw-regret-bound,lem:coverage-bound,lem:offDE-bound,lem:dropped-state-bound,lem:bound_P_and_d}) from \cref{sec:algorithm} to present the following epoch-wise guarantee.
\begin{lemma}
    \label{lem:per-epoch-guarantee}
    For any $m$, any policies $\set{\pi_c}_{c\in \cC}$, and $\delta\in(0,1/2)$, with probability at least $1-\delta/M$,
    \begin{align*}
        \En_{c\sim \cD}\brk*{\abs*{\widehat{V}_{m}^{1}(\pi_c,c)-\Vstar(\pi_c,c)}} &\le \frac{1}{20} \En_{c\sim \cD} \brk*{\reghatm[m-1](\pi_c,c)} + 77e\sqrt{ H^6S^4A^3 \cdot \cE_m}. 
    \end{align*}
\end{lemma}

\paragraph{Proof sketch of \cref{lem:per-epoch-guarantee}}
For simplicity, in this proof sketch, we assume the true reward distribution is known.
For this, we first apply the celebrated local simulation lemma (\cref{local simulation}) in reinforcement learning to relate the divergence of the value functions to the stepwise divergences as the following. Under any context $c$,
\begin{align*}
    \abs*{\widehat{V}_{m}^{1}(\pi_c,c)-\Vstar(\pi_c,c)} \leq \sum\nolimits_{h,s,a}\widehat{d}^{h}_{m}(s,a;\pi_c,c) \Dhel{\widehat{P}^{h}_{m}(s,a;c)}{\Pstar(s,a;c)}.
\end{align*}
Then we can exchange the estimated occupancy measure $\widehat{d}^{h}_{m}(s,a;\pi_c,c)$ by the trusted occupancy measure through \cref{lem:dropped-state-bound}, that is,
\begin{align*}
    \widehat{d}^{h}_{m}(s,a;\pi_c,c) \Dhel{\widehat{P}^{h}_{m}(s,a;c)}{\Pstar(s,a;c)} &\leq \widetilde{d}^{h}_{m}(s,a;\pi_c,c) \Dhel{\widehat{P}^{h}_{m}(s,a;c)}{\Pstar(s,a;c)}\\
    &\quad +  32e \sqrt{H^4S^2A\cdot \cE_m}  +  \widehat{\mathrm{reg}}_{m-1}(\pi;c)/(90HSA).
\end{align*}

Then by coverage guarantee \cref{lem:coverage-bound}, for any $h,s,a$, we can bound
\begin{align*}
    \revindent\widetilde{d}^{h}_{m}(s,a;\pi_c,c) \Dhel{\widehat{P}^{h}_{m}(s,a;c)}{\Pstar(s,a;c)} \\
    &\leq  \frac{\gammam}{e^2H}\cdot\metapol(c,\pol) \ocmtil(s,a;\pol,c)\cdot D_{\mathrm{H}}^2\prn{\Pstar(s,a;c),\Phatm(s,a;c)} + \cE_m'.
\end{align*}
If the assumption in \cref{lem:bound_P_and_d} is satisfied, then
by \cref{lem:bound_P_and_d} and the definition of $\metapol$, we have
\begin{align*}
    \revindent[.4]\sum\limits_{h,s,a}\frac{\gammam}{e^2H}\cdot\metapol(c,\pol) \ocmtil(s,a;\pol,c)\cdot D_{\mathrm{H}}^2\prn{\Pstar(s,a;c),\Phatm(s,a;c)} \\
    &\leq \sum\limits_{h,s,a}\frac{\gammam}{H}\cdot\metapol(c,\pol) \ocmm(s,a;\pol,c)\cdot D_{\mathrm{H}}^2\prn{\Pstar(s,a;c),\Phatm(s,a;c)} \\
    &\leq \frac{\gammam}{H}\cdot  \sum\limits_{h} \bE_{\pi\sim p_{m}^{h}(c)}\left[\Enmpic[\Mstar][\pi] \left[ \Dhels{\widehat{P}^{h}_{m}(s_1^h,a_1^h;c)}{\Pstar(s_1^h,a_1^h;c)} \right]\right].
\end{align*}
If the assumptions in \cref{lem:bound_P_and_d} are not satisfied, we have similar control as well (see the full proof \cref{app:regret-analysis} for details). 
Altogether with taking expectation on $c$, by the offline density estimation bound from \cref{lem:offDE-bound}, we have
\begin{equation*}
\pushQED{\qed} 
    \En_{c\sim \cD}\brk*{\abs*{\widehat{V}_{m}^{1}(\pi_c,c)-\Vstar(\pi_c,c)}} \leq  \En \brk*{\reghatm[m-1](\pi_c,c)}/40 + 39e\sqrt{ H^6S^4A^3 \cdot \cE_m}. \qedhere
    \popQED
\end{equation*}
A revised version of the regret analysis (\cref{lem:yunzong}) in \citet{simchi2020bypassing}, which relates the epoch-wise guarantee to the regret estimation error, can be found in \cref{app:regret-analysis}.
Combining \cref{lem:per-epoch-guarantee,lem:yunzong}, we arrive at the following general regret guarantee.
\begin{theorem}\label{tabular main thm}
    The outputs $\set{\pi_t}_{t\in [T]}$ of \cref{alg:mainalg}
    satisfies with probability at least $1-\delta$ that
    \begin{align*}
        \Reg(T)\lesssim \sum\nolimits_{m=1}^{N}(\tau_m - \tau_{m-1})\cdot  \sqrt{ H^8S^4A^3 \cdot \cE_m}
    \end{align*}
    for $\delta\in (0,1/2)$.
\end{theorem}

\section{Extension: Reward-free Reinforcement Learning for CMDPs}
\label{sec:rf-cmdp}
\newcommand{\pistarR}{\pistar[c,R]}

In this section, we introduce the application of \cref{alg:mainalg} in the task of reward-free reinforcement learning in (stochastic) CMDPs. All proofs in this section will be deferred to \cref{app:rf-cmdp}.

\paragraph{Reward-free reinforcement learning \citep{jin2020reward}} 
Reward-free reinforcement learning aims to efficiently explore the environment without relying on observed rewards. By doing so, it aims to enable the computation of a nearly optimal policy for any given reward function, utilizing only the trajectory data collected during exploration and without needing further interaction with the environment. This approach holds particular significance in scenarios where reward functions are refined over multiple iterations to encourage specific behaviors through trial and error, such as in constrained RL formulations. In such cases, repeatedly applying the same RL algorithm with varying reward functions can prove to be highly inefficient regarding sample usage, underscoring the efficiency of reward-free reinforcement learning.

\paragraph{Problem formulation} The major differences between the regret minimization setting (\cref{sec:problem-setup}) and the reward-free reinforcement learning are that in the latter, no reward signals are observed during the interaction, and the goal of the latter is to output a CMDP prediction $\Mhat$ whose value functions are close to the underlying true CMDP $\Mstar$ for any reward distributions. To accommodate such a change, for any model $M = \set{\Pm^h,\Rm^h}_{h\in [H]}$ and reward distribution $R = \set{R^h}_{h\in [H]}$, we define $M(\cdot;R) = \set{\Pm^h, R^h}_{h\in [H]}= \set{\Pm^h(c), R^h(c)}_{h\in [H],c\in \cC}$ to be model $M$ with the reward distribution part replaced by $R$. Thus in the reward-free reinforcement learning setting, the underlying true model satisfies $\Mstar = \Mstar(\cdot; 0)$ where $0$ is used to denote the reward distribution that is constantly $0$.  
\looseness=-1

For any model $M$ reward distribution $R$, context $c$ and policy $\pi$, we use $M(\pi,c;R)$ to denote the distribution of the trajectory $c_1, \pi_1, s_1\ind{1}, a_1\ind{1}, r_1\ind{1}, \dots, s_1\ind{H}, a_1\ind{H}, r_1\ind{H}$ given $\Mstar = M(\cdot;R)$, $c_1=c$, and $\pi_1=\pi$. Also denote the probability and the expectation under $M(\pi,c;R)$ to be $\PmpicR\prn{\cdot}$ and $\EnmpicR\brk{\cdot}$ respectively.
Given reward distribution $R$, any policy $\pi$, state $s$ and action $a$, we define the action value function $\Qstar(s,a;\pi,c,R)$ at layer $h$ and the value function $\Vstar(s;\pi,c,R)$ at layer $h$ under context $c$ and policy $\pi$ as
\begin{equation*}
    \Qstar(s,a;\pi,c,R)=\sum\limits_{j=h}^{H}\EnmpicR[\Mstar] \brk{r_1^{j} \mid{} s_1\ind{h},a_1\ind{h} = s,a } 
    \text{~~and~~}\Vstar[h](s;\pi,c,R)= \max_{a\in \cA} \Qstar(s,a;\pi,c,R).
\end{equation*}

We denote the optimal policy with reward distribution $R$ under context $c$ as $\pistarR$ and abbreviate its value function as $\Vstar[h](\cdot;c,R)$. 
For $h=1$, we denote $\Vstar(c,R) = \Vstar(s^1;c,R)$ and $\Vstar(\pi,c,R) = \Vstar(s^1;\pi,c,R)$.
We also denote $\Vm[M][h]$ as the value functions when $\Mstar = M$.
\begin{assumption}[Realizability for reward-free RL]
    Suppose the learner is given a model class $\cM$ that contains the underlying true model $\Mstar$. Assume all models $M\in \cM$ have $0$ reward.
\end{assumption}
For a given $\veps,\delta>0$ and a model class $\cM$, the goal of the learner is to output a model $\Mhat$ at the end of the interaction such that for any reward distribution $R$ and set of policies $\set{\pi_c}_{c\in \cC}$, the model satisfies
\begin{align}
    \label{ineq:rfrl-objective}
\En_{c\sim \cD}\brk*{\abs*{\Vstar(\pi_c,c,R) - \Vm[\Mhat](\pi_c,c,R)} } \leq \veps
\end{align}
with probability at least $1-\delta$. An algorithm that achieves this objective is called $(\veps, \delta)$-learns the model class $\cM$. 
Then we have the following guarantee from \cref{alg:mainalg}.
\begin{theorem}
    \label{thm:rfrl-main}
If we choose $\tau_1 = T/(2H)$ and $\tau_2 = T/H$,  
the outputs $\Mhat_2$ of \cref{alg:mainalg}
    satisfies the reward-free objective \cref{ineq:rfrl-objective} with probability at least $1-\delta$, with $T$ at most bounded by 
    \begin{align*}
    T \leq O\prn*{  H^7S^4A^3\log (|\cM|/\delta)/\veps^2  }
    \end{align*}
    for $\delta\in (0,1/2)$.
Moreover, the algorithm requires $O(H)$ number of oracle calls to the $\MLE$ oracle.\looseness=-1
\end{theorem}
The proof follows a similar argument of \cref{lem:per-epoch-guarantee}. In addition, we have a matching lower bound up to a $\poly(H, S, A)$ factor adapted from the non-contextual lower bound from \citet{jin2020reward}. 
\begin{theorem}\label{rf-lower}
Fix $\veps\leq 1$, $\delta\leq 1/2$, $H,A\geq 2$. Suppose $S\geq L\log A$ for a large enough universal constant $L$ and $K\geq 0$ large enough. Then, there exists a CMDP class $\cM$ with $|\cS|=S$, $|\cA|=A$, $|\cM|=K$, and horizon $H$ and a distribution $\mu$ on $\cM$ such that any algorithm $\alg$ that $(\veps/24,\delta)$-learns the class $\cM$ satisfies $\En_{M\sim \mu}\En_{M,\alg} [T] \gtrsim \log |\cM|/\veps^2$,
    where $T$ is the number of trajectories required by the algorithm $\alg$ to achieve $(\veps/24,\delta)$ accuracy and $\En_{M,\alg}[\cdot]$ is the expectation under the interaction between the algorithm $\alg$ and model $M$. 
\end{theorem}

\section{Discussion}
\label{sec:discussion}

In this paper, we establish a reduction from stochastic CMDPs to offline density estimation. In this section, we discuss some potential future research directions.

\paragraph{Extension to low rank CMDPs} Low rank MDPs represent a significant extension to tabular MDPs, as explored in various studies \citep{bradtke1996linear, melo2007q, jiang2017contextual, agarwal2020flambe}. Linear MDPs are typically the first step beyond tabular MDPs. Extending our approach to linear CMDPs would be a substantial achievement. The primary challenge lies in identifying the trusted transitions within linear CMDPs. The current construction for tabular CMDPs does not readily apply here because it does not utilize the low-rank structure. \looseness=-1

\paragraph{Extension to model-free learning} Our approach is model-based. However, model-free methods are often more practical for real-world applications. The main challenge lies in effectively balancing exploration and exploitation using only the value functions, as opposed to our method which depends on the occupancy measure.

\paragraph{More efficient oracles} In this paper, we focus on offline density estimation oracles due to the necessity of a small Hellinger distance between the estimated model and the true model for our approach. An offline regression oracle would only provide a 2-norm distance guarantee, which is inadequate for our proof to go through. Hence, it is interesting to explore whether a reduction from CMDPs to offline regression is feasible.

\section*{Acknowledgements}
Jian Qian acknowledges support from ARO through award W911NF-21-1-0328 and from the Simons Foundation and NSF through award DMS-2031883.

\bibliography{refs}

\appendix

\section{Technical Tools}
\label{app:technical}

\subsection{Maximum Likelihood Estimation for Density Estimation}

\begin{example}[MLE for finite model class]
\label{lem:mle-random-design}
Let $\cM$ be a finite model class and the MLE estimator $\Mhat$ be defined by
\begin{align*}
\Mhat = \argmax_{M\in \cM} \prod_{i=1}^n \Pmpic[M][\pi_i][c_i] \prn*{\set{s_i\ind{h}, a_i\ind{h}, r_i\ind{h}}_{h\in [H]}}.
\end{align*}
For any $\delta\in (0,1/2)$, we have with probability at least $1-\delta$, 
\begin{align*}
    \En_{c\sim \cD,\pi\sim p(c)}\brk*{\Dhels{\Mhat(\pi,c)}{\Mstar(\pi,c)} } \lesssim   \frac{\log (|\cM|/\delta)}{n} .
\end{align*}
\end{example}

\begin{proof}[\pfref{lem:mle-random-design}]
For any $\delta\in (0,1/2)$, by Lemma C.2 of \citet{foster2024online}, we have with probability at least $1-\delta/2$,
\begin{align*}
    \sum_{i=1}^n \Dhels{\Mhat(\pi_i,c_i)}{\Mstar(\pi_i,c_i)} \lesssim \log (|\cM|/\delta).
\end{align*}
Then by Lemma A.3 of \citet{foster2021statistical}, we have with probability at least $1-\delta/2$,
\begin{align*}
        \En_{c\sim \cD,\pi\sim p(c)}\brk*{\Dhels{\Mhat(\pi,c)}{\Mstar(\pi,c)} } \lesssim \sum_{i=1}^n \Dhels{\Mhat(\pi_i,c_i)}{\Mstar(\pi_i,c_i)} + \log (|\cM|/\delta).
\end{align*}
Then by union bound, we obtain the desired result.
\end{proof}

\subsection{Information Theory}

\begin{lemma}[Lemma B.4 of \citet{foster2022complexity}]
    \label{lem:hellinger-with-variance}
    Let $\mathbb{P}$ and $\mathbb{Q}$ be two distributions on space $\chi$. Let $h:\chi\rightarrow R$ be a function. Then we have:
    $$|\bE_{\mathbb{P}}[h]-\bE_{\mathbb{Q}}[h]|\le \sqrt{2^{-1}(\bE_{\mathbb{P}}[h^2]+\bE_{\mathbb{Q}}[h^2])D_{H}^2(\mathbb{P},\mathbb{Q})}.$$
\end{lemma}

\subsection{Reinforcement Learning}

\begin{lemma}[Lemma F.3 of \cite{foster2021statistical}]
\label{local simulation}
    Let $M=\set{\Pm^h,\Rm^h}_{h\in [H]}$ and $\Mbar= \set{\Pm[\Mbar]^{h},\Rm[\Mbar]^{h}}_{h\in [H]}$ be two CMDPs. For any policy $\pi$ and context $c$, we have
    \begin{align*}
        &\Vm(\pi,c) - \Vm[\Mbar](\pi,c) \\
        &= \sum\limits_{h=1}^{H} \Enmpic[\Mbar] \brk*{ \prn*{ \Pm^h(s_1^{h+1}|s_1^h,a_1^h;c) - \Pm[\Mbar]^h(s_1^{h+1}|s_1^h,a_1^h;c) } \Vm[M][h+1](s_1^{h+1};\pi,c)      }\\
        &\quad + \sum\limits_{h=1}^{H} \Enmpic[\Mbar] \brk*{ \En_{r^h\sim \Rm(s_1^h,a_1^h;c)} [r^h] - \En_{r^h\sim \Rm[\Mbar](s_1^h,a_1^h;c)} [r^h]  }\\
        &\leq \sum\limits_{h=1}^{H} \sum\limits_{s,a}\Enmpic[\Mbar] \brk*{\mathbbm{1}( s_1^h,a_1^h= s,a)  } \prn*{  \Dhel{\Pm^h(s,a;c)}{\Pm[\Mbar]^h(s,a;c)} + \Dhel{\Rm^h(s,a;c)}{\Rm[\Mbar]^h(s,a;c)} }.
    \end{align*}

\end{lemma}

\section{Proofs from \cref{sec:algorithm}}
\label{app:algorithm}
In this section, we present the proofs for \cref{lem:igw-regret-bound,lem:coverage-bound,lem:offDE-bound,lem:dropped-state-bound,lem:bound_P_and_d}.

\begin{proof}[\pfref{lem:igw-regret-bound}]
    We fix an arbitrary context $c$ throughout the proof. 
Let $u(\lambda)\ldef \sum_{\pi\in \polCov} 1/\prn{\lambda +  \etam\cdot\reghatm[m-1](\pi)   }$.
Since for any $h\in [H]$, $\pimhat[m-1]\in \polCov$, then $u(\lambda) \geq 1/\prn{\lambda +  \etam\cdot\reghatm[m-1](\pimhat[m-1])   } = 1/\lambda$. On the other hand, $u(\lambda) \leq (SA+1)/\lambda$. Moreover, $u(\lambda)$ is clearly monotonically decreasing with $u(0) = \infty$ and $u(SA+1) \leq 1$. Thus there exists $\lambdam\in (0,SA+1]$ such that $u(\lambdam)=1$ as we desired.

Now we have the regret is bounded by
\begin{align*}
    \En_{\pi\sim p_{m}^{h}(c)}\brk*{\widehat{\mathrm{reg}}_{m-1}(\pi;c)}=\sum_{\pi\in\Pi_{m,c}^{h}}\frac{\widehat{\mathrm{reg}}_{m-1}(\pi;c)}{\lambda_{m,c}^h+\eta_m\cdot\widehat{\mathrm{reg}}_{m-1}(\pi;c)}\le\sum_{\pi\in\Pi_{m}^{h}(c)}\frac{1}{\eta_m}\leq \frac{2SA}{\eta_m}.
\end{align*}
Finally, we recall $\eta_m=\gamma_m/(720e^3H^5S^3A^2)$ and $\gammam = \sqrt{\frac{H^6S^4A^3}{\cE_m}}$ plug this bound in, then we have
\begin{align*}
    \En_{\pi\sim p_{m}^{h}(c)}\brk*{\widehat{\mathrm{reg}}_{m-1}(\pi;c)} \leq 1440e^3\cdot \sqrt{H^4S^4A^3\cdot \cE_m}.
\end{align*}
\end{proof}

\begin{proof}[\pfref{lem:offDE-bound}]
By definition of  \cref{def:offline-DE-oracle}.
\end{proof}

\begin{proof}[\pfref{lem:dropped-state-bound}]
For any $m,\pi,c,h,s,a$, by the definition of $\widetilde{d}_{m}^{h}(s,a;\pi,c)$, we the difference between $\widehat{d}_{m}^{h}(s,a;\pi,c)$ and $\widetilde{d}_{m}^{h}(s,a;\pi,c)$ are the parts of occupancy measures that do not go through the trusted transitions, i.e.,
\begin{align*}
    \revindent[.3] \widehat{d}_{m}^{h}(s,a;\pi,c) - \widetilde{d}_{m}^{h}(s,a;\pi,c)   \\
    & = \sum_{j=1}^{h-1}\sum_{(s^j,a^j,s^{j+1})\notin \widetilde{\mathcal{T}}_{m}^{j}(c)}\widetilde{d}^{j}_{m}(s^j,a^j;\pi,c)\widehat{P}^{j}_{m}(s^{j+1}|s^j,a^j;c)\widehat{P}^{j+1:h}_{m+1}(s|s^{j+1};\pi,c)\pi^{h}(s,a),
\end{align*}
where $\widehat{P}^{j+1:h}_{m+1}(s|s^{j+1};\pi,c)$ is the estimated transition probability from $s^{j+1}$ at step $j+1$ to $s$ at step $h$ under policy $\pi$ and context $c$. Then since $(s^j,a^j,s^{j+1})\notin \widetilde{\mathcal{T}}_{m}^{j}(c)$, we have 
\begin{align*}
    &\sum_{j=1}^{h-1}\sum_{(s^j,a^j,s^{j+1})\notin \widetilde{\mathcal{T}}_{m}^{j}(c)}\widetilde{d}^{j}_{m}(s^j,a^j;\pi,c)\widehat{P}^{j}_{m}(s^{j+1}|s^j,a^j;c)\widehat{P}^{j+1:h}_{m+1}(s|s^{j+1};\pi,c)\pi^{h}(s,a) \\
    &\leq \sum_{j=1}^{h-1}\sum_{(s^j,a^j,s^{j+1})\notin \widetilde{\mathcal{T}}_{m}^{j}(c)}\widetilde{d}^{j}_{m}(s^j,a^j;\pi,c)\widehat{P}^{j}_{m}(s^{j+1}|s^j,a^j;c) \\
    &\leq \sum_{j=1}^{h-1}\sum_{(s^j,a^j,s^{j+1})\notin \widetilde{\mathcal{T}}_{m}^{j}(c)} \frac{SA+\eta_m\widehat{\mathrm{reg}}_{m-1}(\pi;c)}{\zeta_m} \\
    &\leq \frac{hS^2A (SA+\eta_m\widehat{\mathrm{reg}}_{m-1}(\pi;c))}{\zetam}.
\end{align*}
Recall the choice of $\zetam=\frac{\gammam}{8eH(H+1)^2}$, $\etam=\frac{\gammam}{720e^3H^5S^3A^2}$ and $\gammam=\sqrt{\frac{H^6S^4A^3}{\cE_m}}$, we have
\begin{align*}
    \widehat{d}_{m}^{h}(s,a;\pi,c) - \widetilde{d}_{m}^{h}(s,a;\pi,c)\leq 32e \sqrt{H^4S^2A\cdot \cE_m}  + \frac{1}{90HSA} \widehat{\mathrm{reg}}_{m-1}(\pi;c).
\end{align*}
\end{proof}

\begin{proof}[\pfref{lem:bound_P_and_d}]
We prove iteratively between the objective 
\begin{align*}
    \widetilde{d}^{h}_{m}(s,a;\pi,c)\le \prn*{1+\frac{1}{H}}^{2(h-1)}d^h_{m}(s,a;\pi,c)
\end{align*}
and the following claim:
For any $h$ and any $(s,a,s')\in \Ttil(c)$, we have 
\begin{align*}
  \Phatm(s'|s,a;c)\leq \prn*{1+\frac{1}{H}}^{2} \Pstar(s'|s,a;c).
\end{align*}
First, we show that for any $h\in [H]$,
\begin{align}
  \revindent\forall \pi,s,a,~\ocmtil[m](s,a;\pi,c) \leq \prn*{1+\frac{1}{H}}^{2(h-1)} \ocmm(s,a;\pi,c)\notag\\
  &~~\tto~~\forall (s,a,s')\in\Ttil(c),~\Phatm(s'|s,a;c) \leq \prn*{1+ \frac{1}{H}}^2 \Pstar(s'|s,a;c).\label{ineq:d-to-P}
\end{align}
For this, we note by \cref{lem:hellinger-with-variance} that for any $h,s,a,s'$,
\begin{align}
  &\Phatm(s'|s,a;c) \\
  &\leq \Pstar(s'|s,a;c) + \sqrt{2^{-1}(\Phatm(s'|s,a;c) +  \Pstar(s'|s,a;c) ) \Dhels{\Ber(\Phatm(s'|s,a;c))}{\Ber(\Pstar(s'|s,a;c))} }\notag\\
  &\leq \Pstar(s'|s,a;c) + \sqrt{2^{-1}(\Phatm(s'|s,a;c) +  \Pstar(s'|s,a;c) ) \Dhels{\Phatm(s,a;c)}{\Pstar(s,a;c)} }, \label{ineq:hellinger-variance}
\end{align}
where the second inequality is by data-processing inequality \citep{polyanskiy2014lecture}.
Then by AM-GM, we have
\begin{align*}
  \revindent\sqrt{2^{-1}(\Phatm(s'|s,a;c) +  \Pstar(s'|s,a;c) ) \Dhels{\Phatm(s,a;c)}{\Pstar(s,a;c)} } \\
  &\leq \frac{1}{4H}(\Phatm(s'|s,a;c) +  \Pstar(s'|s,a;c) ) + H \cdot \Dhels{\Phatm(s,a;c)}{\Pstar(s,a;c)}.
\end{align*}
Plug the above back into \cref{ineq:hellinger-variance} and reorganize, we obtain
\begin{align*}
  \Phatm(s'|s,a;c) &\leq \prn*{1+ \frac{1}{H}} \cdot \Pstar(s'|s,a;c) + (H+1) \Dhels{\Phatm(s,a;c)}{\Pstar(s,a;c)}.
\end{align*}
Then multiply both sides by $\ocmtil(s,a;\pi,c)$, we have
\begin{align*}
  \ocmtil(s,a;\pi,c)\prn*{\Phatm(s'|s,a;c) - \prn*{1+ \frac{1}{H}}\Pstar(s'|s,a;c)} &\leq  (H+1)\ocmtil(s,a;\pi,c) \Dhels{\Phatm(s,a;c)}{\Pstar(s,a;c)}.
\end{align*} 
Meanwhile, by the definition of $\pol$, we have
\begin{align*}
  \ocmtil(s,a;\pi,c) \leq \frac{\ocmtil(s,a;\pol,c) }{SA + \etam\cdot\reghatm[m-1](\pol,c)} \cdot (SA + \etam\cdot\reghatm[m-1](\pi,c)).
\end{align*}
Then, by the induction hypothesis, we have
\begin{align*}
  \revindent\frac{\ocmtil(s,a;\pol,c) }{SA + \etam\cdot\reghatm[m-1](\pol,c)} \cdot (SA + \etam\cdot\reghatm[m-1](\pi,c)) \\
  &\leq \frac{e^2\ocmm(s,a;\pol,c) }{SA + \etam\cdot\reghatm[m-1](\pol,c)} \cdot (SA + \etam\cdot\reghatm[m-1](\pi,c)).
\end{align*}
Thus we have further by the definition of $\metapol(c)$ and the assumption that $$\bE_{\pi\sim  \metapol(c)}\Enmpic[\Mstar]\left[
    \Dhels{\widehat{P}^h_{m}(s_1^h,a_1^h;c)}{\Pstar(s_1^h,a_1^h;c)}
    \right] \leq H/\gammam,$$
\begin{align*}
  \revindent[.3](H+1)\ocmtil(s,a;\pi,c) \Dhels{\Phatm(s,a;c)}{\Pstar(s,a;c)} \\
  &\leq e^2(H+1)(SA + \etam\cdot\reghatm[m-1](\pi,c)) \frac{\ocmm(s,a;\pol,c) }{SA + \etam\cdot\reghatm[m-1](\pol,c)} \cdot  \Dhels{\Phatm(s,a;c)}{\Pstar(s,a;c)}\\
  &\leq e^2(H+1)(SA + \etam\cdot\reghatm[m-1](\pi,c)) \metapol(c,\pol)\ocmm(s,a;\pol,c) \Dhels{\Pstar(s,a;c)}{\Phatm(s,a;c)}\\
  &\leq e^2(H+1)(SA + \etam\cdot\reghatm[m-1](\pi,c))\bE_{\pi\sim  \metapol(c)}\Enmpic[\Mstar]\left[
    \Dhels{\widehat{P}^h_{m}(s_1^h,a_1^h;c)}{\Pstar(s_1^h,a_1^h;c)}
    \right]\\
  &\leq  \frac{e^2H(H+1)}{\gammam}(SA + \etam\cdot\reghatm[m-1](\pi,c))\\
  &=  \frac{1}{(H+1)\zetam}(SA + \etam\cdot\reghatm[m-1](\pi,c)),
\end{align*}
where the last equality is by the definition of $\zetam$.
In all, we have shown that 
\begin{align*}
  \ocmtil(s,a;\pi,c)\prn*{\Phatm(s'|s,a;c) - \prn*{1+ \frac{1}{H}}\Pstar(s'|s,a;c)} &\leq  \frac{1}{(H+1)\zetam}(SA + \etam\cdot\reghatm[m-1](\pi,c)).
\end{align*}
Now we prove by contradiction, if for any $(s,a,s')\in\Ttil(c)$, the reverse inequality is true, i.e., $\Phatm(s'|s,a;c) > \prn*{1+ \frac{1}{H}}^2\Pstar(s'|s,a;c)$. Then for any $\pi$,
\begin{align*}
\frac{1}{H+1}   \ocmtil(s,a;\pi,c)\Phatm(s'|s,a;c)&<  \ocmtil(s,a;\pi,c)\prn*{\Phatm(s'|s,a;c) - \prn*{1+ \frac{1}{H}}\Pstar(s'|s,a;c)} \\
&\leq \frac{1}{(H+1)\zetam}(SA + \etam\cdot\reghatm[m-1](\pi,c)).
\end{align*}
This contradicts the definition of $\Ttil(c)$.

For the other direction of \cref{ineq:d-to-P}, we prove for all $h\in [H]$,
\begin{align}
\hspace{1in}&\hspace{-1in}\begin{cases}
  \Phatm(s'|s,a;c) \leq \prn*{1+ \frac{1}{H}}^2 \Pstar(s'|s,a;c) & \forall (s,a,s')\in\Ttil,\\
  \ocmtil(s,a;\pi,c) \leq \prn*{1+\frac{1}{H}}^{2(h-1)} \ocmm(s,a;\pi,c) & \forall \pi,s,a.
\end{cases}\notag\\
&\tto \quad \ocmtil[m][h+1](s,a;\pi,c) \leq \prn*{1+\frac{1}{H}}^{2h} \ocmm[m][h+1](s,a;\pi,c), ~~\forall \pi,s,a.\label{ineq:P-to-d}
\end{align}
This direction is straightforward since
\begin{align*}
  \ocmtil[m][h+1](s,a;\pi,c) &= \sum\limits_{(s',a',s)\in \Ttil} \ocmtil(s,a;\pi,c) \Phatm(s'|s,a;c) \pi^{h+1}(s,a;c) \\
  &\leq \prn*{1+\frac{1}{H}}^{2h} \sum\limits_{(s',a',s)\in \Ttil}  \ocmm(s,a;\pi,c) \Pstar(s'|s,a;c)\pi_{h+1}(s,a;c) \\
  &= \prn*{1+\frac{1}{H}}^{2h}  \ocmm[m][h+1](s,a;\pi,c).
\end{align*}
With the two derivations \cref{ineq:d-to-P} and \cref{ineq:P-to-d}, along with the fact that the initial argument of \cref{ineq:d-to-P} holds by definition for $h=1$. Thus, we conclude the proof.

\end{proof}

\begin{proof}[\pfref{lem:coverage-bound}]
    For any $m,\pi,c,h,s,a$, by AM-GM, we have
    \begin{align}
        \begin{split}
            \revindent[.7]\ocmtil (s,a;\pi,c) \Dhel{\Pstar(s,a;c)}{\Phatm(s,a;c)}\\
            &\leq \frac{2e^2H(SA+\eta_m\widehat{\mathrm{reg}}_{m-1}(\pol;c))}{\gammam\widetilde{d}^{h}_{m}(s,a;\pol,c)}\cdot (\widetilde{d}^h_{m}(s,a;\pi,c))^2 \\
            &\quad+ \frac{\gammam\widetilde{d}^{h}_{m}(s,a;\pol;c) }{2e^2H(SA+\eta_m\widehat{\mathrm{reg}}_{m-1}(\pol;c))}\cdot\Dhels{\Pstar (s,a;c)}{\widehat{P}^{h}_{m}(s,a;c)}.\label{ineq:coverage-1}
        \end{split}
    \end{align}
    By the definition of $\pi^{h,s,a}_{m,c}$, we have futher
    \begin{align}
        \begin{split}
            \revindent\frac{2e^2H(SA+\eta_m\widehat{\mathrm{reg}}_{m-1}(\pol;c))}{\gammam\widetilde{d}^{h}_{m}(s,a;\pol,c)}\cdot (\widetilde{d}^h_{m}(s,a;\pi,c))^2 \\
            &\leq \frac{2e^2H(SA+\eta_m\widehat{\mathrm{reg}}_{m-1}(\pi;c))}{\gammam\widetilde{d}^{h}_{m}(s,a;\pi,c)}\cdot (\widetilde{d}^h_{m}(s,a;\pi,c))^2 \\
            &= \frac{2e^2HSA+ 2e^2H \eta_m\widehat{\mathrm{reg}}_{m-1}(\pi;c)}{\gammam} \cdot \widetilde{d}^h_{m}(s,a;\pi,c).\label{ineq:coverage-2}
        \end{split}
    \end{align}
    Recall the choice of  $\eta_m=\gamma_m/(720e^3H^5S^3A^2)$ and $\gammam = \sqrt{\frac{H^6S^4A^3}{\cE_m}}$, we have
    \begin{align}
        \frac{2e^2HSA+ 2e^2H \eta_m\widehat{\mathrm{reg}}_{m-1}(\pi;c)}{\gammam} \leq  2e^2\sqrt{ \frac{\cE_m}{H^4S^2A}}  + \frac{1}{720 H^4S^3A^2}\reghatm(\pi,c). \label{ineq:coverage-3}
    \end{align}
    Also by the definition of $\metapol(c)$, we have
    \begin{align}
        \begin{split}
            \revindent\frac{\gammam\widetilde{d}^{h}_{m}(s,a;\pol;c) }{2e^2H(SA+\eta_m\widehat{\mathrm{reg}}_{m-1}(\pol;c))}\cdot\Dhels{\Pstar (s,a;c)}{\widehat{P}^{h}_{m}(s,a;c)} \\
            &\leq \frac{\gammam}{e^2H}\cdot\metapol(c,\pol) \ocmtil(s,a;\pol,c)\Dhels{\Pstar(s,a;c)}{\Phatm(s,a;c)}. \label{ineq:coverage-4}
        \end{split}
    \end{align}

Now we plug \cref{ineq:coverage-2,ineq:coverage-3,ineq:coverage-4} back into \cref{ineq:coverage-1} to obtain that 
\begin{align*}
    \revindent[.5]\ocmtil (s,a;\pi,c) \Dhel{\Pstar(s,a;c)}{\Phatm(s,a;c)} \\
    &\leq \prn*{2e^2\sqrt{ \frac{\cE_m}{H^4S^2A}}  + \frac{1}{720 H^4S^3A^2}\reghatm[m-1](\pi,c)}\widetilde{d}^h_{m}(s,a;\pi,c) \\
    &\quad + \frac{\gammam}{e^2H}\cdot\metapol(c,\pol) \ocmtil(s,a;\pol,c)\Dhels{\Pstar(s,a;c)}{\Phatm(s,a;c)}.
\end{align*}
Similar bounds can be obtained replacing $\Dhel{\Pstar(s,a;c)}{\Phatm(s,a;c)}$ with $\Dhel{\Rstar(s,a;c)}{\Rhatm(s,a;c)}$. 
\end{proof}

\section{Proofs from \cref{sec:regret-analysis}}
\label{app:regret-analysis}
\begin{proof}[\pfref{lem:per-epoch-guarantee}]
For this, we first apply the local simulation lemma (\cref{local simulation}) in reinforcement learning to relate the divergence of the value functions to the stepwise divergences as the following. Under any context $c$,
\begin{align*}
    \abs*{\widehat{V}_{m}^{1}(\pi_c,c)-\Vstar(\pi_c,c)} \leq \sum_{h,s,a}\widehat{d}^{h}_{m}(s,a;\pi_c,c) \prn*{ \Dhel{\widehat{P}^{h}_{m}(s,a;c)}{\Pstar(s,a;c)}  + \Dhel{\widehat{R}^{h}_{m}(s,a;c)}{\Rstar(s,a;c)}   }.
\end{align*}
Then we can exchange the estimated occupancy measure $\widehat{d}^{h}_{m}(s,a;\pi_c,c)$ by the trusted occupancy measure through \cref{lem:dropped-state-bound}, that is,
\begin{align}
    \revindent\sum\limits_{h,s,a}\widehat{d}^{h}_{m}(s,a;\pi_c,c) \Dhel{\widehat{P}^{h}_{m}(s,a;c)}{\Pstar(s,a;c)} \notag\\
    &\leq \sum\limits_{h,s,a} \widetilde{d}^{h}_{m}(s,a;\pi_c,c) \Dhel{\widehat{P}^{h}_{m}(s,a;c)}{\Pstar(s,a;c)}\notag\\
    &\quad + \sum\limits_{h,s,a} \prn*{ 32e \sqrt{H^4S^2A\cdot \cE_m}  + \frac{1}{90HSA} \widehat{\mathrm{reg}}_{m-1}(\pi;c)}\notag\\
    &\leq \sum\limits_{h,s,a} \widetilde{d}^{h}_{m}(s,a;\pi_c,c) \Dhel{\widehat{P}^{h}_{m}(s,a;c)}{\Pstar(s,a;c)}\notag\\
    &\quad +  32e \sqrt{H^6S^4A^3\cdot \cE_m}  + \frac{1}{90} \widehat{\mathrm{reg}}_{m-1}(\pi_c;c). \label{ineq:per-epoch-sub-1}
\end{align}

Then by coverage guarantee \cref{lem:coverage-bound}, for any $h,s,a$, we can bound
\begin{align}
    \revindent \sum\limits_{h,s,a}\widetilde{d}^{h}_{m}(s,a;\pi_c,c) \Dhel{\widehat{P}^{h}_{m}(s,a;c)}{\Pstar(s,a;c)} \notag\\
    &\leq \sum\limits_{h,s,a} \frac{\gammam}{e^2H}\cdot\metapol(c,\pol) \ocmtil(s,a;\pol,c)\cdot D_{\mathrm{H}}^2\prn{\Pstar(s,a;c),\Phatm(s,a;c)} \notag\\
    &\quad \quad + \sum\limits_{h,s,a} \prn*{2e^2\sqrt{ \frac{\cE_m}{H^4S^2A}}  + \frac{1}{720 H^4S^3A^2}\reghatm[m-1](\pi,c) }\ocmtil (s,a;\pi,c)\notag\\
    &\leq \sum\limits_{h,s,a} \frac{\gammam}{e^2H}\cdot\metapol(c,\pol) \ocmtil(s,a;\pol,c)\cdot D_{\mathrm{H}}^2\prn{\Pstar(s,a;c),\Phatm(s,a;c)}\notag \\
    &\quad \quad + 2e^2\sqrt{ \frac{\cE_m}{H^2A}}  + \frac{1}{720 H^3S^2A^2}\reghatm[m-1](\pi_c,c) . \label{ineq:per-epoch-sub-2}
\end{align}
Suppose the assumption in \cref{lem:bound_P_and_d} is satisfied, then
by \cref{lem:bound_P_and_d}, we have
\begin{align}
    \revindent[.4]\sum\limits_{h,s,a}\frac{\gammam}{e^2H}\cdot\metapol(c,\pol) \ocmtil(s,a;\pol,c)\cdot D_{\mathrm{H}}^2\prn{\Pstar(s,a;c),\Phatm(s,a;c)} \notag\\
    &\leq \sum\limits_{h,s,a}\frac{\gammam}{H}\cdot\metapol(c,\pol) \ocmm(s,a;\pol,c)\cdot D_{\mathrm{H}}^2\prn{\Pstar(s,a;c),\Phatm(s,a;c)} \notag\\
    &\leq \frac{\gammam}{H}\cdot  \sum\limits_{h} \bE_{\pi\sim p_{m}^{h}(c)}\left[\Enmpic[\Mstar][\pi] \left[ \Dhels{\widehat{P}^{h}_{m}(s_1^h,a_1^h;c)}{\Pstar(s_1^h,a_1^h;c)} \right]\right]. \label{ineq:per-epoch-sub-3}
\end{align}

The $\Dhel{\widehat{P}^{h}_{m}(s_1^h,a_1^h;c)}{\Pstar(s_1^h,a_1^h;c)}$ in \cref{ineq:per-epoch-sub-1,ineq:per-epoch-sub-2,ineq:per-epoch-sub-3} can be replaced with  $\Dhel{\widehat{R}^{h}_{m}(s_1^h,a_1^h;c)}{\Rstar(s_1^h,a_1^h;c)}$ as well.
If the assumption in \cref{lem:bound_P_and_d} is not satisfied, then we there exists $j$ such that
\begin{align*}
    \bE_{\pi\sim  \metapol(c)}\Enmpic[\Mstar]\left[
        \Dhels{\widehat{P}^h_{m}(s_1^j,a_1^j;c)}{\Pstar(s_1^j,a_1^j;c)}
        \right]> H/\gamma_m.
\end{align*}
This implies 
\begin{align*}
    \abs*{\widehat{V}_{m}^{1}(\pi_c,c)-\Vstar(\pi_c,c)} \leq 1 \leq \frac{\gamma_m}{H} \cdot \bE_{\pi\sim  \metapol(c)}\Enmpic[\Mstar]\left[
        \Dhels{\widehat{P}^h_{m}(s_1^j,a_1^j;c)}{\Pstar(s_1^j,a_1^j;c)}
        \right].
\end{align*}
Thus, altogether, no matter the assumption in \cref{lem:bound_P_and_d} is satisfied or not, we have shown
\begin{align*}
    \abs*{\widehat{V}_{m}^{1}(\pi_c,c)-\Vstar(\pi_c,c)}&\leq 76e \sqrt{H^6S^4A^3\cdot \cE_m}  + \frac{1}{20} \widehat{\mathrm{reg}}_{m-1}(\pi_c;c) \\
    &\quad\quad +  \frac{\gammam}{H}\cdot  \sum\limits_{h} \bE_{\pi\sim p_{m}^{h}(c)}\left[\Enmpic[\Mstar][\pi] \left[ \Dhels{\widehat{P}^{h}_{m}(s_1^h,a_1^h;c)}{\Pstar(s_1^h,a_1^h;c)} \right]\right] \\
    &\quad\quad + \frac{\gammam}{H}\cdot  \sum\limits_{h} \bE_{\pi\sim p_{m}^{h}(c)}\left[\Enmpic[\Mstar][\pi] \left[ \Dhels{\widehat{R}^{h}_{m}(s_1^h,a_1^h;c)}{\Rstar(s_1^h,a_1^h;c)} \right]\right]\\
    &\leq 76e \sqrt{H^6S^4A^3\cdot \cE_m}  + \frac{1}{20} \widehat{\mathrm{reg}}_{m-1}(\pi_c;c) + \gammam \cdot \cE_m \\
    &\leq 77e \sqrt{H^6S^4A^3\cdot \cE_m}  + \frac{1}{20} \widehat{\mathrm{reg}}_{m-1}(\pi_c;c),
\end{align*}
where the second inequality is by the offline density estimation bound from \cref{lem:offDE-bound}.
Taking expectation on $c$, we have
\begin{align*}
    \En_{c\sim \cD}\brk*{\abs*{\widehat{V}_{m}^{1}(\pi_c,c)-\Vstar(\pi_c,c)}} \leq \frac{1}{20} \En \brk*{\reghatm[m-1](\pi_c,c)} + 77e\sqrt{ H^6S^4A^3 \cdot \cE_m}. 
\end{align*}
\end{proof}

\begin{lemma}[\cite{simchi2020bypassing}]
    \label{lem:yunzong}
    Let $\veps_1,\dots,\veps_N$ be $N$ positive values. Suppose for any $m>0$ and an arbitrary policy set $\set{\pi_c}_{c\in \cC}$, we have:
    \begin{align}
        \bE_{c\sim \cD}[|\widehat{V}^1_{m}(\pi_c,c)-\Vstar(\pi_c,c)|]\le \frac{1}{20}\bE_{c\sim \cD}[\reghatm[m-1](\pi_c,c)]+\varepsilon_{m}. \label{ineq:inf-oec-regret-comparison-1}
    \end{align}
    Then for any $m>0$,
    \begin{align}
        \bE_{c\sim \cD}[\reg(\pi_c,c)]\le\frac{10}{9}\cdot \bE_{c\sim \cD}[\reghatm(\pi_c,c)]+\delta_m,  \label{ineq:inf-oec-regret-comparison-2} \\
        \bE_{c\sim \cD}[\reghatm(\pi_c,c)]\le\frac{9}{8}\cdot \bE_{c\sim \cD}[\reg(\pi_c,c)]+\delta_m, \label{ineq:inf-oec-regret-comparison-3} 
    \end{align}
where $\delta_1=2\varepsilon_1+\frac{1}{10}$ and $\delta_m=\frac{1}{9}\delta_{m-1}+\frac{20}{9}\varepsilon_m$ for any $m\geq 2$.    
\end{lemma}

\begin{proof}[\pfref{lem:yunzong}]
We present the proof here for completeness.
By \cref{ineq:inf-oec-regret-comparison-1}, we have that for all $m\geq 0$ and $\pi_c$,
\begin{align}
    \revindent[.5]\bE_{c\sim \cD}[\reg(\pi_c,c)] - \bE_{c\sim \cD}[\reghatm(\pi_c,c)] \notag \\
    &= \bE_{c\sim \cD}[\Vstar(\pistar ,c) -  \Vhatm(\pistar, c)] + \bE_{c\sim \cD}[\Vhatm(\pistar, c) -\Vhatm(\pimhat, c) ] \notag\\
    &\quad + \bE_{c\sim \cD}[\Vhatm(\pi_c, c) -\Vstar(\pi_c, c) ] \notag\\
    &\leq \frac{1}{20}\bE_{c\sim \cD}[\reghatm[m-1](\pistar,c)]+ \frac{1}{20}\bE_{c\sim \cD}[\reghatm[m-1](\pi_c,c)]  +2\varepsilon_{m}. \label{ineq:inf-oec-regret-comparison-4}
\end{align}
Symmetrically, we have
\begin{align}
    \revindent[.5]\bE_{c\sim \cD}[\reghatm(\pi_c,c)] - \bE_{c\sim \cD}[\reg(\pi_c,c)]  \notag\\
    &\leq \frac{1}{20}\bE_{c\sim \cD}[\reghatm[m-1](\pimhat,c)]+ \frac{1}{20}\bE_{c\sim \cD}[\reghatm[m-1](\pi_c,c)]  +2\varepsilon_{m}.\label{ineq:inf-oec-regret-comparison-5}
\end{align}
Now we inductively show for $m=1,2,\dots$ that \cref{ineq:inf-oec-regret-comparison-2} and \cref{ineq:inf-oec-regret-comparison-3} hold.
For $m=1$, since $\bE_{c\sim \cD}[\reg(\pi_c,c)],\bE_{c\sim \cD}[\reghatm[1](\pi_c,c)]\leq 1$ for all  $\pi$, then we have from \cref{ineq:inf-oec-regret-comparison-4} and \cref{ineq:inf-oec-regret-comparison-5}
\begin{align*}
    &\bE_{c\sim \cD}[\reg(\pi_c,c)] \leq  \bE_{c\sim \cD}[\reghatm[1](\pi_c,c)] + 2\veps_1 +  \frac{1}{10} \mathrm{~~and~~} \\
    &\bE_{c\sim \cD}[\reghatm[1](\pi_c,c)]\leq  \bE_{c\sim \cD}[\reg(\pi_c,c)] + 2\veps_1 +  \frac{1}{10}.
\end{align*}
Hence we have shown \cref{ineq:inf-oec-regret-comparison-2} and \cref{ineq:inf-oec-regret-comparison-3} for $m=1$ and  $\delta_1 = 2\veps_1 + \frac{1}{10}$.
Now, suppose \cref{ineq:inf-oec-regret-comparison-2} and \cref{ineq:inf-oec-regret-comparison-3} holds for all $1,2,\dots,m-1$. Plugging \cref{ineq:inf-oec-regret-comparison-3} for $m-1$ into the right hand side of \cref{ineq:inf-oec-regret-comparison-4}, we have
\begin{align*}
    \revindent[.5]\bE_{c\sim \cD}[\reg(\pi_c,c)] -  \bE_{c\sim \cD}[\reghatm(\pi_c,c)] \\
    &\leq  \frac{1}{20} \bE_{c\sim \cD}[\reghatm[m-1](\pistar,c)] + \frac{1}{20} \bE_{c\sim \cD}[\reghatm[m-1](\pi_c,c)]  + 2\veps_{m}\\
    &\leq \frac{1}{20} \prn*{\frac{9}{8} \bE_{c\sim \cD}[\reg(\pistar,c)]+ \delta_{m-1}} + \frac{1}{20} \prn*{\frac{9}{8} \bE_{c\sim \cD}[\reg(\pi_c,c)] + \delta_{m-1}} + 2\veps_{m}\\
    &= \frac{1}{20} \delta_{m-1} + \frac{1}{20}\prn*{ \frac{9}{8} \bE_{c\sim \cD}[\reg(\pi_c,c)] + \delta_{m-1} } +  2\veps_m,
\end{align*}
where the last equality is by $\bE_{c\sim \cD}[\reg(\pistar,c)]=0$.
Thus, reorganizing the terms, we have
\begin{align*}
    \bE_{c\sim \cD}[\reg(\pi_c,c)] \leq \frac{10}{9} \bE_{c\sim \cD}[\reghatm(\pi_c,c)] + \delta_m,
\end{align*}
where $\delta_m = \frac{1}{9}\delta_{m-1} + \frac{20}{9} \veps_m$. Thus we have shown \cref{ineq:inf-oec-regret-comparison-2} for $m$. Then plugging \cref{ineq:inf-oec-regret-comparison-3} for $m-1$ into the right hand side of \cref{ineq:inf-oec-regret-comparison-5}, we have
\begin{align*}
    \revindent[.5]\bE_{c\sim \cD}[\reghatm(\pi_c,c)] - \bE_{c\sim \cD}[\reg(\pi_c,c)] \\
    &\leq \frac{1}{20} \bE_{c\sim \cD}[\reghatm[m-1](\pimhat,c)] + \frac{1}{20}  \bE_{c\sim \cD}[\reghatm[m-1](\pi_c,c)]  + 2\veps_{m}\\
    &\leq \frac{1}{20} \prn*{\frac{9}{8} \bE_{c\sim \cD}[\reg(\pimhat,c)] + \delta_{m-1}} + \frac{1}{20} \prn*{\frac{9}{8} \bE_{c\sim \cD}[\reg(\pi_c,c)] + \delta_{m-1}}   + 2\veps_m.
\end{align*}
Furthermore, by \cref{ineq:inf-oec-regret-comparison-2} for $m$ we have $\bE_{c\sim \cD}[\reg(\pimhat,c)] \leq \frac{10}{9} \bE_{c\sim \cD}[\reghat(\pimhat,c)] + \delta_m = \delta_m$. Plug this in the aforementioned inequality, we have
\begin{align*}
    \bE_{c\sim \cD}[\reghatm(\pi_c,c)] - \bE_{c\sim \cD}[\reg(\pi_c,c)] &\leq\frac{1}{20} \prn*{\frac{9}{8} \delta_m + \delta_{m-1}} + \frac{1}{20} \prn*{\frac{9}{8} \bE_{c\sim \cD}[\reg(\pi_c,c)] + \delta_{m-1}}   + 2\veps_m.
\end{align*}
Reorganizing the terms, in turn gives
\begin{align*}
    \bE_{c\sim \cD}[\reghatm(\pi_c,c)] \leq \frac{9}{8} \bE_{c\sim \cD}[\reg(\pi_c,c)] + \delta_m,
\end{align*}
where recall $\delta_m = \frac{1}{9} \delta_{m-1} + \frac{20}{9}\veps_m$. This proves  \cref{ineq:inf-oec-regret-comparison-3} for $m$, which completes our induction.

\end{proof}

\begin{proof}[\pfref{tabular main thm}]
    Let $\cE_0=1$. Then we have by \cref{lem:per-epoch-guarantee}, \cref{ineq:inf-oec-regret-comparison-1} holds with $\veps_m = L\sqrt{H^6S^4A^3\cE_m}$ for $L>0$ large enough for all $m>0$. 
    Combining \cref{lem:yunzong,lem:igw-regret-bound}, we have
    \begin{align*}
     \En_{c\sim \cD, \pi\sim \metapol(c)} [\reg(\pi,c)] &\lesssim  \bE_{c\sim \cD, \pi\sim \metapol(c)}[\reghatm[m-1](\pi_c,c)]  + \sum\limits_{i=0}^{m-1} \frac{1}{9^{m-i}}\sqrt{ H^6S^4A^3 \cdot \cE_i}
    \end{align*}
    Then by \cref{lem:igw-regret-bound}, we have further
    \begin{align*}
        \bE_{c\sim \cD, \pi\sim \metapol(c)}[\reghatm[m-1](\pi_c,c)]&\lesssim \sqrt{ H^6S^4A^3 \cdot \cE_m}.
    \end{align*}
    Hence, we have
    \begin{align*}
        \En_{c\sim \cD, \pi\sim \metapol(c)} [\reg(\pi,c)] &\lesssim  \sum\limits_{i=0}^{m} \frac{1}{9^{m-i}}\sqrt{ H^6S^4A^3 \cdot \cE_i}.
    \end{align*}
    In all, we can obtain the following regret bound with probability at least $1-\delta$,
    \begin{align*}
        \sum\nolimits_{t=1}^{T}
        \Reg(T) &= \sum\limits_{t=1}^{T} \bE_{c_t\sim \cD, \pi_t\sim \metapol[m(t)][h(t)](c_t)}[\mathrm{reg}(\pi_t,c_t)]\\
        &= \sum\limits_{h,m}\En_{c\sim \cD, \pi\sim \metapol(c)} [\reg(\pi,c)] \cdot(\tau_m- \tau_{m-1}) \\
        &\lesssim \sum\limits_{m=1}^{N}(\tau_m - \tau_{m-1})\cdot  \sqrt{ H^8S^4A^3 \cdot \cE_m} ,
    \end{align*}
    where the last step takes a union bound on the offline oracle guarantees.
\end{proof}

\begin{proof}[\pfref{thm:loglogT}]
Without loss of generality, assume that $T/H > 1000$. Since we are choosing $\tau_m = 2(T/H)^{1-2^{-m}}$, we first note that since $\tau_m\leq T/H$, we have $(T/H)^{2^{-m}}\geq 2$. Then we have
\begin{align*}
\tau_{m} - \tau_{m-1}  &= 2(T/H)^{1-2^{-m}} - 2(T/H)^{1-2^{1-m}}\\
&= 2(T/H)^{1-2^{1-m}} (T^{2^{-m}}-1)\\
&\geq 2(T/H)^{1-2^{1-m}} = \tau_{m-1}.
\end{align*}
This implies $\tau_m - \tau_{m-1} \geq \frac{1}{2}\tau_m$ for $m\in [N]$.
Then by \cref{tabular main thm}, we have with probability at least $1-\delta$,
\begin{align*}
    \sum\nolimits_{t=1}^{T}
        \Reg(T) &\leq \sum\limits_{m=1}^{N}(\tau_m - \tau_{m-1})\cdot  \sqrt{ H^8S^4A^3 \cdot \cE_m}\\
        &\lesssim   \sqrt{H^8S^4A^3 \log (|\cM|N/\delta)} \cdot \sum\limits_{m=2}^{N}  \frac{\tau_m - \tau_{m-1}}{\sqrt{\tau_{m-1} - \tau_{m-2}}}   + \tau_1 \sqrt{H^8S^4A^3}\\
        &\lesssim \sqrt{H^8S^4A^3 \log (|\cM|N/\delta)} \cdot \sum\limits_{m=2}^{N}  \frac{\tau_m }{\sqrt{\tau_{m-1}}} + \sqrt{H^7S^4A^3 T}\\
        &\lesssim \sqrt{H^8S^4A^3 \log (|\cM|N/\delta)} N \sqrt{T/H} = \sqrt{ H^7S^4A^3 T\cdot \log(|\cM|\log \log T/\delta)\log\log T},
\end{align*}
where the last inequality follows from 
\begin{align*}
\frac{\tau_m}{\sqrt{\tau_{m-1}}} \leq \frac{\tau_m}{\sqrt{(\tau_{m-1}+1)/2}}  \leq \frac{2(T/H)^{1-2^{-m}}}{(T/H)^{\frac{1}{2}(1-2^{1-m})}} = 2\sqrt{T/H}
\end{align*}
and $N = O(\log\log T)$. So the number of oracle calls is $O(H\log\log T)$.

\end{proof}

\begin{proof}[\pfref{thm:logT}]
Since we are choosing $\tau_m  = 2^m$, we have $N= O(\log T)$. So the number of oracle calls is $O(H\log T)$.
Meanwhile, by \cref{tabular main thm}, we have with probability at least $1-\delta$,
\begin{align*}
    \sum\nolimits_{t=1}^{T}
        \Reg(T) &\leq \sum\limits_{m=1}^{N}(\tau_m - \tau_{m-1})\cdot  \sqrt{ H^8S^4A^3 \cdot \cE_m}\\
        &\lesssim  \sqrt{H^8S^4A^3 \log (|\cM|N/\delta)} \prn*{1 + \sum\limits_{m=3}^{N} \frac{2^{m-1}}{\sqrt{2^{m-2}}} } \\
        &\lesssim  \sqrt{H^8S^4A^3 \log (|\cM|N/\delta)} \cdot 2^{N/2} = \sqrt{ H^7S^4A^3 T\cdot \log(|\cM| \log T/\delta)}.
\end{align*}

\end{proof}

\section{Proofs from \cref{sec:rf-cmdp}}
\label{app:rf-cmdp}
\begin{proof}[\pfref{thm:rfrl-main}]
The reward distributions $\widehat{R}_m$ are $0$ for $m=0,1$ since $\wh{R}_0$ are set to be $0$ and the model class $\cM$ consists of models with constantly $0$ reward. The regret estimations $\reghatm[m-1]$ are all $0$ for $m=1,2$. Thus we apply
the component-wise guarantees 
(\cref{lem:coverage-bound,lem:offDE-bound,lem:dropped-state-bound,lem:bound_P_and_d}) to $\wh{P}_2, \wh{R}_2$ where $\wh{R}_2=0$.  Concretely, from \cref{lem:offDE-bound} we have with probability at least $1-\delta$,
\begin{align}
    \label{ineq:offline-guarantee-rf}
    \bE_{c\sim \cD,\pi\sim \metapol[2](c)}
        \left[\Enmpic[\Mstar] \left[ 
        \Dhels{\widehat{P}^{h}_{2}(s_1^h,a_1^h;c)}{\Pstar(s_1^h,a_1^h;c)}
    \right]\right]
    \lesssim  \cE_2,
\end{align}
with the offline $\MLE$ oracle.
From \cref{lem:dropped-state-bound}, we have for any $\pi,c,h,s,a$,
\begin{align}
    \widehat{d}_{2}^{h}(s,a;\pi,c) - \widetilde{d}_{2}^{h}(s,a;\pi,c)\le 32e \sqrt{H^4S^2A\cdot \cE_2} .\label{ineq:dropped-state-rf}
\end{align}
From \cref{lem:bound_P_and_d}, we have if for a context $c$ and all $h\in[H]$,
$$\bE_{\pi\sim  \metapol[2](c)}\Enmpic[\Mstar]\left[
    \Dhels{\widehat{P}^h_{2}(s_1^h,a_1^h;c)}{\Pstar(s_1^h,a_1^h;c)}
    \right]\le H/\gamma_2.$$
Then for the same $c$ and all $\pi,h,s,a$, we have
\begin{align}
\widetilde{d}^{h}_{2}(s,a;\pi,c)\le \prn*{1+1/H}^{2(h-1)}d^h_{2}(s,a;\pi,c). \label{ineq:bound-d-rf}
\end{align}
Finally, from \cref{lem:coverage-bound}, we have for any $\pi,c,h,s,a$, 
\begin{align}
    \label{ineq:coverage-rf}
    \begin{split}
        \revindent[.7]\ocmtil[2] (s,a;\pi,c)\cdot D_{\mathrm{H}}\prn{\Pstar(s,a;c),\Phatm[2](s,a;c)}
        \\
        &\leq \frac{\gammam[2]}{e^2H}\cdot\metapol[2](c,\pol[2]) \ocmtil[2](s,a;\pol[2],c)\cdot D_{\mathrm{H}}^2\prn{\Pstar(s,a;c),\Phatm[2](s,a;c)} \\
        &\quad +     2e^2\sqrt{ \frac{\cE_2}{H^4S^2A}} \cdot  \ocmtil[2] (s,a;\pi,c).
    \end{split} 
\end{align}
Now, we are ready to prove our claim following similar derivations from the proof of \cref{lem:per-epoch-guarantee}. Concretely, we have for any set of policies $\set{\pi_c}_{c\in\cC}$ and any context $c$, by local simulation lemma (\cref{local simulation})
\begin{align*}
    \abs*{\Vstar(\pi_c,c,R) - \Vhatm[2](\pi_c,c,R)}  \leq  \sum_{h,s,a}\widehat{d}^{h}_{2}(s,a;\pi_c,c) \Dhel{\widehat{P}^{h}_{2}(s,a;c)}{\Pstar(s,a;c)}.
\end{align*} 
Then by \cref{ineq:dropped-state-rf}, we have
\begin{align}
    \revindent[.5]\sum\limits_{h,s,a}\widehat{d}^{h}_{2}(s,a;\pi_c,c) \Dhel{\widehat{P}^{h}_{2}(s,a;c)}{\Pstar(s,a;c)} \notag\\
    &\leq \sum\limits_{h,s,a} \widetilde{d}^{h}_{2}(s,a;\pi_c,c) \Dhel{\widehat{P}^{h}_{2}(s,a;c)}{\Pstar(s,a;c)} +  32e \sqrt{H^6S^4A^3\cdot \cE_2} \label{ineq:per-epoch-sub-1}
\end{align}

\textbf{Case I:} If for a context $c$ and all $h\in[H]$,
$$\bE_{\pi\sim  \metapol[2](c)}\Enmpic[\Mstar]\left[
    \Dhels{\widehat{P}^h_{2}(s_1^h,a_1^h;c)}{\Pstar(s_1^h,a_1^h;c)}
    \right]\le H/\gamma_2.$$
Then by \cref{ineq:coverage-rf,ineq:bound-d-rf},
\begin{align*}
    \revindent\sum\limits_{h,s,a} \widetilde{d}^{h}_{2}(s,a;\pi_c,c) \Dhel{\widehat{P}^{h}_{2}(s,a;c)}{\Pstar(s,a;c)}\\
    &\leq \frac{\gammam[2]}{e^2H}\cdot\metapol[2](c,\pol[2]) \ocmtil[2](s,a;\pol[2],c)\cdot D_{\mathrm{H}}^2\prn{\Pstar(s,a;c),\Phatm[2](s,a;c)} \\
    &\quad +     2e^2\sqrt{ \frac{\cE_2}{H^4S^2A}}  \cdot \ocmtil[2] (s,a;\pi_c,c)\\
    &\leq \frac{\gammam[2]}{H}\cdot\metapol[2](c,\pol[2]) \ocmm[2](s,a;\pol[2],c)\cdot D_{\mathrm{H}}^2\prn{\Pstar(s,a;c),\Phatm[2](s,a;c)} \\
    &\quad +     2e^2\sqrt{ \frac{\cE_2}{H^4S^2A}}  \cdot \ocmtil[2] (s,a;\pi_c,c)
\end{align*}
Thus we have
\begin{align*}
    \revindent[.5]\abs*{\Vstar(\pi_c,c,R) - \Vhatm[2](\pi_c,c,R)} \\
    &\lesssim \sqrt{H^6S^4A^3\cdot \cE_2} + \frac{\gammam[2]}{H}\cdot \sum\limits_{h=1}^{H}\bE_{\pi\sim  \metapol[2](c)}\Enmpic[\Mstar]\left[\Dhels{\widehat{P}^h_{2}(s_1^h,a_1^h;c)}{\Pstar(s_1^h,a_1^h;c)}
    \right].
\end{align*}

\textbf{Case II}: If for a context $c$ there exists $j\in[H]$ such that 
$$\bE_{\pi\sim  \metapol[2][h](c)}\Enmpic[\Mstar]\left[
    \Dhels{\widehat{P}^j_{2}(s_1^j,a_1^j;c)}{\Pstar(s_1^j,a_1^j;c)}
    \right]> H/\gamma_2.$$
    Then 
    \begin{align*}
        \revindent\abs*{\Vstar(\pi_c,c,R) - \Vhatm[2](\pi_c,c,R)} \\
        &\leq 1 \leq  \frac{\gammam[2]}{H}\cdot \sum\limits_{h=1}^{H}\bE_{\pi\sim  \metapol[2](c)}\Enmpic[\Mstar]\left[\Dhels{\widehat{P}^h_{2}(s_1^h,a_1^h;c)}{\Pstar(s_1^h,a_1^h;c)}
        \right].
    \end{align*}
    Combine Case I and II, we have
    \begin{align*}
        \revindent[.5]\abs*{\Vstar(\pi_c,c,R) - \Vhatm[2](\pi_c,c,R)} \\
        &\lesssim \sqrt{H^6S^4A^3\cdot \cE_2} + \frac{\gammam[2]}{H}\cdot \sum\limits_{h=1}^{H}\bE_{\pi\sim  \metapol[2](c)}\Enmpic[\Mstar]\left[\Dhels{\widehat{P}^h_{2}(s_1^h,a_1^h;c)}{\Pstar(s_1^h,a_1^h;c)}
        \right].
    \end{align*}
    Then take expectation with respect to $c$ together with \cref{ineq:offline-guarantee-rf}  we obtain
\begin{align*}
    \En_{c\sim \cD}\brk*{\abs*{\Vstar(\pi_c,c,R) - \Vhatm[2](\pi_c,c,R)} } &\lesssim   \sqrt{H^6S^4A^3\cdot \cE_2} + \gammam[2] \cE_2 \\
    &\lesssim \sqrt{\frac{H^7S^4A^3 \log (|\cM|/\delta)}{T}}\\
    &\lesssim \veps,
\end{align*}
where the second inequality is by $\cE_2 =H \log (|\cM|/\delta)/T$ and the last inequality holds when
\begin{align*}
    T \geq \Omega\prn*{  \frac{H^7S^4A^3\log (|\cM|/\delta)}{\veps^2}  }.
\end{align*}
Thus, the reward-free objective  \cref{ineq:rfrl-objective} is satisfied with probability at least $1-\delta$ with 
\begin{align*}
T \leq O\prn*{  \frac{H^7S^4A^3\log (|\cM|/\delta)}{\veps^2}  }.
\end{align*}
\end{proof}

\begin{lemma}[Lemma D.2 of \citet{jin2020reward}]
\label{lem:jin-rf}
    Fix $\veps\leq 1$, $\delta\leq 1/2$, $H,A\geq 2$, and suppose that $S\geq L\log A$ for a universal constant $L$. Then consider the trivial context space $\cC = \set{c_0}$, there exists a model class $\cM = \set{M_j}_{j\in \cJ}$, with $|\cS|=S$, $|\cA|=A$, $|\cJ|\leq e^{SA}$, and horizon $H$ and a distribution $\mu$ on $\cM$ such that any algorithm $\alg$ that $(\veps/12,\delta)$-learns the class $\cM$ satisfies
    \begin{align*}
        \En_{M\sim \mu}\En_{M,\alg} [T] \gtrsim \frac{SA}{\veps^2}, 
    \end{align*}
    where $T$ is the number of trajectories required by the algorithm $\alg$ to achieve $(\veps/12,\delta)$ accuracy and $\En_{M,\alg}[\cdot]$ is the expectation under the interaction between the algorithm $\alg$ and model $M$. 
\end{lemma}

\begin{proof}[Proof of \cref{rf-lower}].
Let $ n = \log K /(SA)$, then we consider the context space $\cC = \set{c_1,...,c_{n}}$ and the i.i.d. distribution $\cD$ on the context space being uniform. 
Denote the model class obtained from \cref{lem:jin-rf} by $\wb{\cM} = \set{\Mbar_j(c_0)}_{j\in \cJ} $.
Let $J = \set{j_1,...,j_n} \in \cJ^n$ be an index. 
Then we consider the model class $\cM = \set{M_J}_{J\in \cJ^n}$, where $M_J(c_i) = \Mbar_{j_i}(c_0)$, that is, the model class $\cM$ is on each context $c_i\in \cC$ an independent $\wb{\cM}$. 
We first have that the size of the model class is bounded by $|\cM| = |\cJ|^n \leq e^{nSA} \leq K$. Then we have for any algorithm $\alg$ that $(\veps/24,\delta)$-learns the class $\cM$, it must have $(\veps/12,\delta)$-learns the class $\cM(c_i) = \set{M_J(c_i)}_{J\in \cJ^n} = \set{\Mbar_j(c_0)}_{j\in \cJ}$ for at least half of the contexts $c_i\in \cC$ by Markov's inequality. This, in turn, combined with \cref{lem:jin-rf} gives that there exists a distribution $\mu$ on the model class $\cM$ such that 
\begin{align*}
    \En_{M\sim \mu}\En_{M,\alg} [T] \gtrsim \frac{n}{2}\frac{SA}{\veps^2} \gtrsim \frac{\log K}{\veps^2} , 
\end{align*}
where $T$ is the number of trajectories required by the algorithm $\alg$ to achieve $(\veps/12,\delta)$ accuracy and $\En_{M,\alg}[\cdot]$ is the expectation under the interaction between the algorithm $\alg$ and model $M$. Thus concludes our proof.

\end{proof}

\section{Computation}\label{app:compute}
\newcommand{\rbar}{\bar{r}}
\newcommand{\vrbar}{\bar{\mathbf{r}}}
\newcommand{\sbar}{\bar{s}}
\newcommand{\abar}{\bar{a}}
\renewcommand{\hbar}{\bar{h}}
\newcommand{\vdtil}{\wt{\mathbf{d}}}
\newcommand{\vdhat}{\wh{\mathbf{d}}}

In this section, we show that for any $m,c,h,s,a$, the policy  $\pi^{h,s,a}_{m,c}$ (\pref{line:policy-covering}) and the trusted transitions $\widetilde{\mathcal{T}}^h_m(c)$ (\cref{trusted-occu}) can be computed efficiently through linear programming. For simplicity, we fix a context $c$ throughout this section and omit its dependence. We first note that if  $\widetilde{\mathcal{T}}^{j}_m$ for $j\leq h-1$ and $\pi^{h,s,a}_{m}$ are computed, then  $\widetilde{\mathcal{T}}^h_m$ can be computed in $\poly(HSA)$ time. To see this, for any $(s,a,s')$, we have by the definition of $\pi_m^{h,s,a}$ that
\begin{align*}
    \max_{\pi}\frac{\widetilde{d}^h_{m}(s,a;\pi)}{SA+\etam\cdot \reghat_{m-1}(\pi)} =\frac{\widetilde{d}^h_{m}(s,a;\pi_m^{h,s,a})}{SA+\etam\cdot \reghat_{m-1}(\pi_m^{h,s,a})}.
\end{align*}
Then $(s,a,s')\in \widetilde{\mathcal{T}}^h_m$ iff 
\begin{align*}
    \frac{\widetilde{d}^h_{m}(s,a;\pi_m^{h,s,a})}{SA+\etam\cdot \reghat_{m-1}(\pi_m^{h,s,a})}\widehat{P}^{h}_{m}(s'|s,a) \geq  \frac{1}{\zetam},
\end{align*}
where the left-hand side can be computed given $\widetilde{\mathcal{T}}^{j}_m$ for $j\leq h-1$ in $\poly(HSA)$ time. Thus we only need to show how to compute $\pi_m^{h,s,a}$. Assume we want to compute $\pi_{m}^{\hbar,\sbar,\abar}$ for $\hbar\in [H],\sbar\in \cS,\abar\in\cA$, given  $\widetilde{\mathcal{T}}^{h}_m$ for $h\leq \hbar-1$. Let $\rbar_m^h(s,a) = \En_{r^h\sim \Rhatm[m-1]}[r^h]$ be the mean reward for model $\Mhatm[m-1]$ for $h,s,a$. Let $\vrbar_m = (\rbar_m^h(s,a))_{h,s,a}$ be the vector of mean rewards for the model  $\Mhatm[m-1]$. We consider the following linear fractional program with the following decision variables: 
\begin{align*}
    \vd_m^{\hbar} =
    \begin{pmatrix}
        \vdtil_m^{\hbar}\\
        \vdhat_{m-1}
    \end{pmatrix},
    ~~\text{where}~\vdtil^{\hbar}_{m}=(\widetilde{d}_{h,s,a,m})_{h,s,a\in[\Bar{h}]\times\mathcal{S}\times\mathcal{A}
},\ \vdhat_{m-1}=(\widehat{d}_{h,s,a,m-1})_{h,s,a\in[H]\times\mathcal{S}\times\mathcal{A}}.
\end{align*}
We will use the variable $\vdtil^{\hbar}_{m}$ to simulate the trusted occupancy measures and 
$\vdhat_{m-1}$ to simulate the estimated occupancy measures from $\widehat{M}_{m-1}$ by linear constraints. Concretely, consider the following linear fractional program:
\begin{align*}
\begin{split}
        \max_{\vd_m^{\hbar}}&~~\frac{\widetilde{d}_{\hbar,\sbar,\abar,m}}{SA+\eta_m(\widehat{V}^1_{m-1}-\langle \vdhat_{m-1},\vrbar_{m-1}\rangle )},\\
\mathrm{subject\ to:}&~~
\begin{cases}
         \widehat{d}_{h,s,a,m-1}\ge 0, &\forall\ h,s,a\in [H]\times\cS\times\cA, \\
     \sum_{a}\widehat{d}_{1,s,a,m-1}= \mathbbm{1}(s=s^1),&\forall s\in \cS,\\
     \sum_{s,a}\widehat{d}_{h,s,a,m-1}\widehat{P}^h_{m-1}(s'|s,a)=\sum_{a}\widehat{d}_{h+1,s',a,m-1}, &\forall\ h,s'\in [H]\times\cS,\\
     \widetilde{d}_{h,s,a,m}\ge 0,&\forall\ h,s,a\in [\Bar{h}]\times\mathcal{S}\times\mathcal{A}\\
\sum_{a}\widetilde{d}_{1,s,a,m}=\mathbbm{1}(s=s^1),&\forall s\in \cS,\\
\sum_{s,a,s'\in\widetilde{\mathcal{T}}^{h-1}_{m}}\widetilde{d}_{h-1,s,a,m}\widehat{P}_{m}^{h-1}(s'|s,a)=\sum_{a}\widetilde{d}_{h,s',a,m} &\forall\ h\le\Bar{h},s'\in\mathcal{S}.
\end{cases}
\end{split}
\end{align*}
This is a linear fractional program of $HSA+\Bar{h}SA$ decision variables with $HSA+HS+\Bar{h}SA+\Bar{h}S$ constraints.  It is clear from the linear constraints that this program simulates the MDPs $\Mhatm[m-1]$ and $\Mhatm$, and the program obtains the right objective.
Then, for this linear fractional program, we apply the Charnes-Cooper transformation \cite{charnes1962programming} to transform it into a linear program. After the transformation, we can apply existing tools for solving linear programs (e.g., \citet{lee2019solving}) to solve for $\vdtil_m^{\hbar}$ which encodes the policy $\pi_m^{\hbar,\sbar,\abar}$ in $\poly(HSA)$ time.

\end{document}